\documentclass{article}

\PassOptionsToPackage{numbers,compress}{natbib}

\usepackage[utf8]{inputenc}
\usepackage[T1]{fontenc}

\usepackage[preprint]{neurips_2026}

\usepackage{microtype}
\usepackage{graphicx}
\usepackage{subcaption}
\usepackage{booktabs}
\usepackage{hyperref}
\usepackage{float}
\usepackage{url}
\usepackage{nicefrac}
\usepackage{amsmath}
\usepackage{amssymb}
\usepackage{amsfonts}
\usepackage{mathtools}
\usepackage{amsthm}
\usepackage{tcolorbox}
\usepackage[dvipsnames]{xcolor}
\usepackage[capitalize,noabbrev]{cleveref}
\usepackage{wrapfig}

\hypersetup{
    colorlinks=true,
    allcolors=BrickRed
}

\theoremstyle{plain}
\newtheorem{theorem}{Theorem}[section]
\newtheorem{proposition}[theorem]{Proposition}
\newtheorem{lemma}[theorem]{Lemma}

\theoremstyle{definition}
\newtheorem{definition}[theorem]{Definition}
\newtheorem{assumption}[theorem]{Assumption}
\theoremstyle{remark}

\definecolor{main-bg}{RGB}{240, 246, 255}
\definecolor{main-frame}{RGB}{74, 111, 165}
\definecolor{mygreen}{HTML}{008b5f}

\title{Test-Time Personalization: A Diagnostic Framework and Probabilistic Fix for Scaling Failures}

\author{%
  Linhai Zhang \\
  King's College London\\
  \texttt{linhai.zhang@kcl.ac.uk}
  \And
  Yulan He\thanks{Corresponding author.} \\
  King's College London \\
  The Alan Turing Institute \\
  \texttt{yulan.he@kcl.ac.uk}
}

\begin{document}

\maketitle

\begin{abstract}
Existing approaches to LLM personalization focus on constructing better personalized models or inputs, while treating inference as a single-shot process.
In this work, we study \textbf{Test-Time Personalization (TTP)} along an unexplored axis: scaling inference-time computation by sampling $N$ candidates from a personalized policy model and selecting the best with a personalized reward model.
%We develop a theoretical framework, validated empirically, that characterizes when and why TTP succeeds or fails.
We prove that oracle selection yields expected utility growing logarithmically with the number of sampled candidates, establishing a theoretical ceiling for test-time scaling.
However, standard reward models fail to realize this potential. To diagnose why, we derive a unified scaling law that decomposes any reward model's Best-of-$N$ curve into four measurable quantities and reveals two failure modes, \emph{user-level collapse} (near-constant prediction for some users) and \emph{query-level reward hacking} (negative correlation with true quality for some queries).
Guided by this law, we propose a probabilistic personalized reward model whose learned variance  effectively mitigates both failure modes.
Experiments confirm both elements of our framework: TTP delivers consistent scaling across multiple policy models and personalized text generation tasks, and our scaling law closely matches observed scaling curves across reward-model variants.
\end{abstract}

%---- BEGIN sections/1_introduction.tex ----
\section{Introduction}

Large language models excel across diverse tasks, yet they predominantly produce \emph{one-size-fits-all} responses that ignore individual user preferences~\cite{liu2025survey}.
This limitation has motivated growing research interest in LLM personalization.
Existing approaches fall into three categories.
\emph{Personalized Prompting} augments input context with retrieved user history~\cite{richardson2023integrating,salemi-etal-2024-lamp}. 
\emph{Personalized Adaptation} directly fine-tunes model parameters on user data~\cite{tan-etal-2024-democratizing,zhang-etal-2025-proper}. 
\emph{Personalized Alignment} combines multi-objective reward models with user-specific weights~\cite{jang2024personalized,aroca-ouellette2025aligning}.
Despite their differences, these approaches share a common paradigm: they focus on constructing better personalized models or inputs, while treating inference as a single-shot generation process.

Meanwhile, scaling test-time computation has emerged as a powerful axis for improving LLM performance, particularly in reasoning tasks~\cite{wu2025inference, snell2025scaling, shen2025thinking}.
Recent work has begun to combine test-time scaling with personalization, but only along two axes: \emph{scaling user interactions} to learn each user's preferences online~\cite{qu2025tpop}, and \emph{scaling reward-model reasoning} via generative reward models~\cite{zhang2026pgenrm}.
A third, equally natural axis, \emph{scaling the policy model itself}, has not been studied for personalization.
This raises a natural question: \textbf{\emph{Can we improve personalization by scaling test-time computation for a weak personalized model?}}

A straightforward instantiation of test-time personalization is parallel sampling~\cite{salemi-etal-2024-lamp, kumar2024longlamp}: a personalized policy generates $N$ candidates per query, and a personalized reward model selects the best.
We first establish the theoretical promise of this approach: when the number of candidates increases, the expected utility of the optimal sample grows logarithmically (Section~\ref{sec:theory_1}). 
This provides a theoretical ceiling for what test-time personalization can achieve.
Realizing this ceiling in practice, however, is far from trivial.
As Figure~\ref{fig:intro} shows, both a global reward model trained on pooled data and a per-user reward model trained on individual histories perform near random selection across $N$, leaving the oracle gap wide open.
\textbf{\emph{What prevents standard reward models from scaling?}}

\begin{wrapfigure}{r}{0.38\textwidth}
    \vspace{-10pt}
    \centering
    \includegraphics[width=0.36\textwidth]{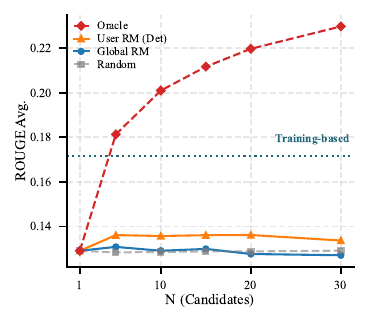}
    \caption{Test-time personalization on LaMP-4: News Headline Generation. Oracle shows logarithmic scaling that surpasses training-based baselines, while standard Reward Models (RMs) fail to scale, performing close to, or worse than, random selection.}
    \label{fig:intro}
    \vspace{-12pt}
\end{wrapfigure}

To diagnose this gap, we develop an analytical framework that connects a reward model's correlation with the golden score to its Best-of-$N$ scaling behavior (Section~\ref{sec:theory_2}).
The analysis surfaces two distinct failure patterns, \emph{user-level collapse} where the reward model degenerates to near-constant predictions for some users and \emph{query-level reward hacking} where the reward model's predictions negatively correlate with quality for some queries. 
We then generalize the oracle scaling law into a \emph{general expression} that forecasts the scaling curve of any reward model from four measurable quantities.
Guided by this expression, we propose a \textbf{probabilistic personalized reward model} that mitigates both failure patterns via learned variance, enabling stable test-time scaling in practice (Section~\ref{sec:theory_3}).

Experiments on two benchmarks covering five personalized text generation tasks confirm both elements of the framework: our probabilistic reward model scales reliably and surpasses training-based baselines on most tasks, and the general expression closely matches observed scaling curves. In summary, our contributions are three-fold:

\begin{itemize}
    \item We introduce Test-Time Personalization (TTP), a new paradigm focused on scaling policy-model compute for personalized text generation.
    \item We detect two failure patterns and a predictive scaling law, a general expression that forecasts any reward model's scaling curve from four measurable quantities.
    \item We propose a probabilistic personalized reward model that mitigates both failure patterns and scales reliably across multiple policy types and tasks.
\end{itemize}

%---- END sections/1_introduction.tex ----

%---- BEGIN sections/2_Preliminaries.tex ----
\section{Preliminaries}
\label{sec:preliminaries}

We focus on \emph{personalized text generation} tasks, where each user has demonstrated stylistic preferences through written examples and the system must mimic these preferences at inference time.

\subsection{Problem Formulation}

Consider a set of users $\mathcal{U}$, where each user $u \in \mathcal{U}$ has an underlying preference function $r_u^*: \mathcal{Q} \times \mathcal{X} \to \mathbb{R}$ that maps a query-response pair $(q, x)$ to a scalar reward reflecting how well $x$ aligns with the user's preferences for query $q$. Each user is associated with historical data $\mathcal{D}_u = \{(q_i, x_i)\}_{i=1}^{n_u}$.

Given a query $q$, a personalized policy $\pi_u$ conditioned on user history generates $N$ candidate responses: $\{x_1, x_2, \ldots, x_N\} \sim \pi_u(\cdot \mid q)$.
A personalized reward model $\hat{r}_u$ trained on $\mathcal{D}_u$ selects the best candidate: $x^* = \arg\max_{x_i} \hat{r}_u(q, x_i)$.
The effectiveness of TTP is measured by the expected true reward of the selected response:
\begin{equation}
    U(N) = \mathbb{E}_{q,\, x_{1:N} \sim \pi_u}\left[r_u^*(q, x^*)\right].
\end{equation}
The \emph{oracle} strategy selects candidates using the true preference function $r_u^*$, yielding the optimal utility $U_{\text{oracle}}(N)$. 
The goal of TTP is to approach this oracle performance. 

\subsection{Experimental Setup}

We use five personalized text generation tasks from LaMP~\cite{salemi-etal-2024-lamp} and LongLaMP~\cite{kumar2024longlamp}, covering news headlines, scholarly titles, abstracts, product reviews, and topical posts. 
The personalized policy follows a retrieval-augmented generation (RAG) recipe~\cite{salemi-etal-2024-lamp}, conditioning generation on user-history examples retrieved per query. 
Since $r_u^*$ is unobservable, we use ROUGE between generated responses and user-written references as a ground-truth-reward proxy. 
For each user, we sample candidate responses from the policy model and compute their ROUGE scores based on the golden response to construct training data. 
We consider two standard reward models: \textbf{Global RM}, trained on pooled data across users, and \textbf{User RM}, trained per user on $\mathcal{D}_u$. 
For evaluation, the policy model generates $N$ candidates per query and different reward models select the best candidate. We report ROUGE scores averaged across users. 
%---- END sections/2_Preliminaries.tex ----

%---- BEGIN sections/3_theory_1.tex ----
\section{The Promise of Test-Time Personalization}
\label{sec:theory_1}

We begin by establishing the theoretical foundation of TTP: given access to an oracle reward function, the expected utility of selected responses grows logarithmically with the number of candidates. 
We then empirically validate this scaling law and show that oracle TTP can surpass training-based methods~\cite{tan-etal-2024-democratizing}.

\subsection{Theoretical Foundation}

Our theoretical analysis establishes a scaling law for oracle selection:

\begin{theorem}[\textbf{Oracle Scaling Law}]
    \label{thm:oracle_scaling}
    Assume the true reward of responses sampled from $\pi_u(\cdot \mid q)$ is sub-Gaussian with mean $\mu_u$ and variance proxy $\sigma_u^2$. Let $x^*_{\mathrm{oracle}} = \arg\max_i r_u^*(q, x_i)$ denote the oracle selection from $N$ i.i.d. samples. Then the expected population-level utility satisfies:
    \begin{equation}
        \bar{U}_{\mathrm{oracle}}(N) = \mathbb{E}_u\left[U_{\mathrm{oracle},u}(N)\right] \leq \bar{\mu} + \bar{\sigma} \cdot c \sqrt{\ln N}
    \end{equation}
    for $N \geq 2$, where $\bar{\mu} = \mathbb{E}_u[\mu_u]$, $\bar{\sigma} = \mathbb{E}_u[\sigma_u]$, and $c > 0$ is a universal constant.
\end{theorem}

The formal proof is provided in Appendix~\ref{app:proof_oracle}. The sub-Gaussian assumption is mild and generally satisfied when the policy generates diverse but bounded-quality responses, which is typical for LLM outputs under temperature sampling; we empirically verify this property on our data in Appendix~\ref{app:assumption_validation}.
The $\sqrt{\ln N}$ form arises from the well-known result on the expected maximum of sub-Gaussian random variables, which grows as $O(\sqrt{\ln N})$ for $N$ i.i.d. samples. 

Theorem~\ref{thm:oracle_scaling} establishes a theoretical ceiling for test-time personalization: even a weak policy (low $\bar{\mu}$) can yield high-quality outputs through scaled sampling, provided the reward model can identify the best candidates. 
The key question becomes whether we can build reward models that approach this oracle performance.

\subsection{Empirical Validation}

We validate the oracle scaling law on two representative personalization tasks: scholarly title generation (LaMP-5) and product review generation (LongLaMP). For each task, we sample $N \in \{1, 5, 10, 15, 20, 30\}$ candidates per query and select the best one using the ground-truth reward (average ROUGE-1 and ROUGE-L scores against reference).

\begin{figure}[h]
    \centering
    \includegraphics[width=0.7\textwidth]{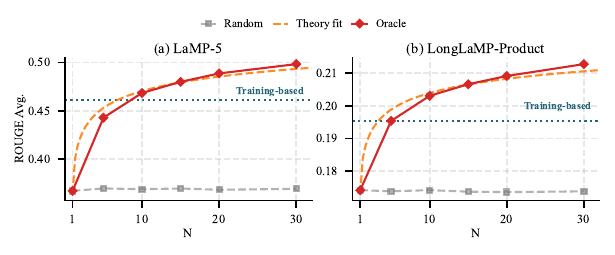}
    \caption{Oracle scaling on (a) LaMP-5 and (b) LongLaMP-Product. Oracle selection (red solid) closely follows the theoretical prediction $\bar{\mu} + \bar{\sigma} c\sqrt{\ln N}$ (orange dashed); the horizontal teal dotted line marks the per-user training-based baseline.}
    \label{fig:oracle_scaling}
\end{figure}

Figure~\ref{fig:oracle_scaling} presents the results.
The oracle scaling curves exhibit clear logarithmic growth, closely matching the theoretical prediction of Theorem~\ref{thm:oracle_scaling}.
Notably, oracle TTP surpasses training-based methods at modest candidate counts ($N \approx 5\text{--}10$), showing that substantial gains are achievable with moderate computational overhead.
These results establish that test-time personalization offers a promising pathway to performance beyond training-based methods, provided we can construct effective reward models.
%---- END sections/3_theory_1.tex ----

%---- BEGIN sections/4_theory_2.tex ----
\section{The Challenges of Test-Time Personalization}
\label{sec:theory_2}

Having established that TTP can substantially improve personalization with oracle reward models, we now investigate whether this potential can be realized with learned reward models. We find that standard approaches fail in surprising ways, then develop a theoretical framework to diagnose the underlying causes.

\subsection{Standard Reward Models Fail to Scale}

We evaluate the two reward model approaches defined in Section~\ref{sec:preliminaries}: Global RM trained on population-level data, and User-specific RM trained individually per user. Figure~\ref{fig:challenges} presents the scaling curves on two representative tasks.

\begin{figure}[h]
    \centering
    \includegraphics[width=0.7\textwidth]{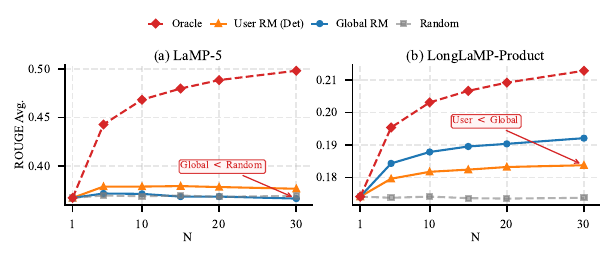}
    \caption{Scaling curves for standard reward models. (a) On LaMP-5, Global RM performs no better than random selection. (b) On LongLaMP-Product Review, User RM underperforms Global RM despite being explicitly trained on personalized user data. Both fall far short of the oracle upper bound.}
    \label{fig:challenges}
\end{figure}

Two unexpected phenomena emerge. First, on LaMP-5, Global RM achieves nearly identical performance to random selection; it fails to provide any meaningful signal for candidate selection. Second, on LongLaMP, User RM underperforms Global RM despite being explicitly personalized. This contradicts the intuition that user-specific training should improve personalization. Both approaches fall far short of the oracle upper bound, indicating that realizing the promise of TTP requires understanding why these standard methods fail.

\subsection{From Correlation to Scaling}

To analyze reward model quality, we introduce the correlation between learned and true rewards:

\begin{definition}[\textbf{Reward Model Correlation}]
    For user $u$, the correlation between a learned reward model $\hat{r}_u$ and the true preference $r_u^*$ is:
    \begin{equation}
        \rho_u = \mathrm{Corr}\left(\hat{r}_u(q,x),\, r_u^*(q,x)\right).
    \end{equation}
\end{definition}

This correlation directly determines scaling behavior:

\begin{lemma}[\textbf{Correlation-Scaling Relationship}]
    \label{lem:corr_scaling}
    Under uniformity assumption, for a reward model with correlation $\rho_u$, the Best-of-N utility satisfies:
    \begin{equation}
        U_u(N) \approx \mu_u + \rho_u \cdot \sigma_u \cdot c\sqrt{\ln N}
    \end{equation}
\end{lemma}

The proof is provided in Appendix~\ref{app:proof_correlation}, and the bivariate-linearity assumption it relies on is empirically verified in Appendix~\ref{app:assumption_validation}. Intuitively, correlation acts as a scaling coefficient: positive $\rho$ yields increasing curves, $\rho \approx 0$ yields flat curves, and negative $\rho$ yields decreasing curves. This framework allows us to diagnose reward model failures by examining correlation distributions.

\subsection{Diagnosing Failure Modes}

Lemma~\ref{lem:corr_scaling} motivates two diagnostic correlations: per-user $\rho_u$ (averaged over queries) and per-query $\rho_q$ (across candidates within a query). Figure~\ref{fig:correlation_diagnostics} reveals two failure modes, each explaining one phenomenon in Figure~\ref{fig:challenges}. \textbf{Query-level reward hacking} ($\rho_q < 0$): on LaMP-5, 46\% of queries exhibit negative correlation under Global RM (37\% under User RM), causing it to systematically prefer worse candidates; this explains why Global RM cannot beat random in Figure~\ref{fig:challenges}(a). \textbf{User-level collapse} ($\rho_u < 0.1$): User RM achieves higher mean correlation than Global RM (0.37 vs 0.24 on LaMP-5; 0.43 vs 0.40 on LongLaMP-Product), but 25\% of LongLaMP-Product users collapse under User RM (vs 0\% under Global RM), where the reward model emits near-constant predictions and provides no useful signal; this explains why User RM falls below Global RM in Figure~\ref{fig:challenges}(b).

\begin{figure}[t]
    \centering
    \includegraphics[width=\textwidth]{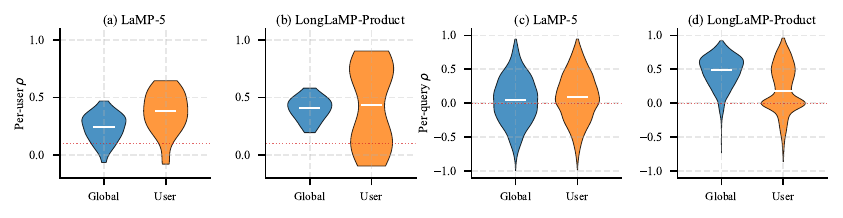}
    \caption{Correlation diagnostics on LaMP-5 and LongLaMP-Product Review. \textbf{(a, b)} Per-user correlation $\rho_u$: User RM achieves higher mean correlation but exhibits a heavy lower tail crossing the red dotted line at $\rho_u\!=\!0.1$ (user-level collapse threshold); Global RM remains tightly concentrated. \textbf{(c, d)} Per-query correlation $\rho_q$: a substantial fraction of queries lie below the red dotted line at $\rho_q\!=\!0$ (query-level reward hacking).}
    \label{fig:correlation_diagnostics}
\end{figure}

\subsection{Unified Scaling Law}

We formalize these failure modes and derive their joint effect on scaling behavior:

\begin{definition}[\textbf{Failure Rates}]
\label{def:failure_rates}
    For a reward modeling approach, we define: 
    \begin{itemize}
        \item \textbf{Collapse rate} $\alpha$: Fraction of users with $\rho_u < 0.1$;
        \item \textbf{Reward hacking rate} $\beta$: Fraction of queries with $\rho_q < 0$.
    \end{itemize}
\end{definition}

\begin{proposition}[\textbf{Unified Scaling Law}]
    \label{prop:unified_law}
    The population-level Best-of-N utility can be approximated as:
    \begin{equation}
        \bar{U}(N) \approx \bar{\mu} + \underbrace{(1-\alpha)}_{\text{non-collapsed}} \cdot \underbrace{\left[(1-\beta)\bar{\rho}_+ - \beta|\bar{\rho}_-|\right]}_{\text{effective correlation}} \cdot \bar{\sigma} \cdot c\sqrt{\ln N}
    \end{equation}
    where $\bar{\rho}_+ > 0$ is the mean correlation on non-hacked queries and $\bar{\rho}_- < 0$ is the mean correlation on hacked queries.
\end{proposition}

This formula reveals the \textbf{personalization-stability trade-off}: Global RM achieves stability (low $\alpha$ and $\beta$) by averaging across users, but this same averaging limits its ability to capture individual preferences. 
User RM achieves higher correlation for many users but suffers from both collapse and reward hacking due to limited per-user data.

Proposition~\ref{prop:unified_law} (proof in Appendix~\ref{app:unified_law}) suggests that an effective approach must simultaneously achieve: (1) high base correlation through personalization, (2) low collapse rate for user-level stability, and (3) low reward hacking rate for query-level stability. We empirically validate the formula in Section~\ref{sec:experiments_main} and Appendix~\ref{app:scaling_law_validation}: the four measured quantities forecast observed Best-of-$N$ scaling curves at relMAE below $3\%$ across all tasks. In the next section, we show that probabilistic reward modeling is the unique design that achieves the three desiderata above.
%---- END sections/4_theory_2.tex ----

%---- BEGIN sections/5_theory_3.tex ----
\section{Probabilistic Reward Modeling}
\label{sec:theory_3}

The analysis in Section~\ref{sec:theory_2} reveals that effective TTP requires high correlation, low collapse rate $\alpha$, and low reward hacking rate $\beta$. We now show that probabilistic reward modeling achieves all three by introducing a learned variance $\sigma^2$ that absorbs uncertainty rather than forcing premature commitment.

\subsection{Theoretical Analysis}

Consider the Gaussian negative log-likelihood (NLL) loss
\begin{equation}
\label{eq:nll_loss}
    \mathcal{L}_{\text{NLL}} = \tfrac{1}{2}\log \sigma^2(x) + \frac{(y - \mu(x))^2}{2\sigma^2(x)},
\end{equation}
whose gradient with respect to $\mu$ is $(\mu - y) / \sigma^2$. The learned $\sigma^2$ thus acts as an adaptive per-input weight, giving two complementary effects.

\begin{lemma}[Gradient Buffering]
\label{lem:gradient_buffering}
For inputs where $y$ is poorly predictable from $x$ (collapse-prone users), increasing $\sigma^2$ attenuates the gradient on $\mu$, preventing the constant-prediction equilibrium that traps deterministic models. This reduces the collapse rate: $\alpha^{\text{prob}} < \alpha^{\text{det}}$.
\end{lemma}

\begin{lemma}[Implicit Regularization]
\label{lem:implicit_regularization}
For inputs with inconsistent training signals, NLL drives $\sigma^2$ large, down-weighting their contribution to the gradient on $\mu$. This prevents overfitting to spurious patterns that produce negative per-query correlation, reducing the hacking rate: $\beta^{\text{prob}} < \beta^{\text{det}}$.
\end{lemma}

Together (proofs in Appendix~\ref{app:proof_lemma51_52}), the two lemmas give probabilistic User RM the unique combination required by Proposition~\ref{prop:unified_law}: low $\alpha$, low $\beta$, and high $\bar{\rho}_{+}$. This produces the monotonic Best-of-$N$ scaling we observe empirically (Section~\ref{sec:experiments_main}), in contrast to the flat or non-monotonic curves of deterministic models when failure modes dominate.

\subsection{Implementation}

We extend a language model backbone with two heads predicting per-input mean and variance:
\begin{equation}
    \mu(x) = \text{Sigmoid}(\text{MLP}_\mu(\mathbf{h})), \qquad \sigma^2(x) = \text{Softplus}(\text{MLP}_\sigma(\mathbf{h})),
\end{equation}
where $\mathbf{h}$ is the masked-mean-pooled hidden state from the backbone, which we adapt per user via LoRA for parameter efficiency.

The training objective combines the NLL loss in Eq.~\eqref{eq:nll_loss} with a margin loss that sharpens ranking among high-quality candidates, the regime that matters for Best-of-$N$ selection:
\begin{equation}
    \mathcal{L}_{\text{contrast}} = \sum_{\substack{i,j:\, y_i > y_j \\ y_i > \tau}} \max(0, m - (\mu_i - \mu_j)),
\end{equation}
where $\tau$ is a high-score threshold and $m$ is the margin, giving $\mathcal{L} = \mathcal{L}_{\text{NLL}} + \lambda \mathcal{L}_{\text{contrast}}$. At inference, we score candidates by $\mu$ alone; the variance is used only at training time. Backbone, hyperparameters, and prompt templates appear in Appendix~\ref{app:implementation}, an ablation across architectural components and the loss decomposition is in Appendix~\ref{app:ablation}, and computational-cost measurements (training and inference timing) are in Appendix~\ref{app:compute-efficiency}.

%---- END sections/5_theory_3.tex ----

%---- BEGIN sections/6_experiments.tex ----
\section{Experiments}
\label{sec:experiments}

We evaluate Test-Time Personalization on five personalized text-generation tasks, validating both the predictive scaling law of Proposition~\ref{prop:unified_law} and the practical effectiveness of probabilistic RM.

\begin{figure*}[t]
    \centering
    \includegraphics[width=\textwidth]{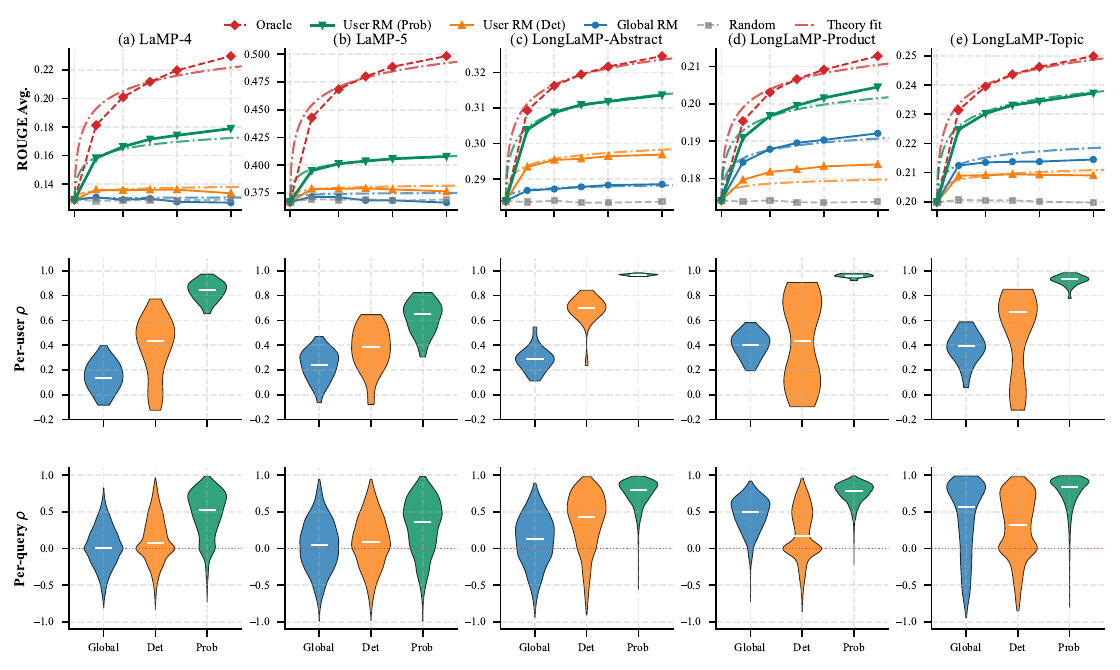}
    \caption{Main TTP results under the RAG policy across five tasks (a)--(e). \textbf{Top row}: observed Best-of-$N$ ROUGE (solid lines with markers) and theory predictions (dash-dot lines), where each method's prediction is computed by substituting its four diagnostic quantities $(\alpha, \beta, \bar{\rho}_{+}, \bar{\rho}_{-})$ into Proposition~\ref{prop:unified_law} (oracle uses Theorem~\ref{thm:oracle_scaling}); theory and observation agree at relMAE $<$ 3\%. \textbf{Middle row}: per-user correlation $\rho_u$ violin distributions. \textbf{Bottom row}: per-query correlation $\rho_q$ violin distributions; the red dotted line marks $\rho_q\!=\!0$. Probabilistic User RM (\textcolor{mygreen}{green}) is the only method whose curve scales monotonically, and the right-shifted, tightened correlation distributions confirm that this gain is driven by the elimination of user-level collapse and query-level reward hacking.}
    \label{fig:main_results}
\end{figure*}

\subsection{Setup}

\paragraph{Datasets.}
We use five tasks from LaMP~\cite{salemi-etal-2024-lamp} (4-News Headline Generation, 5-Scholarly Title Generation) and LongLaMP~\cite{kumar2024longlamp} (Abstract Generation, Product Review Writing, Topic Writing).
Following~\citet{tan-etal-2024-democratizing}, we select the top-$30$ users with the longest profiles for LaMP and the top-$20$ for LongLaMP.

\paragraph{Policy Families.}
Our main results use a retrieval-augmented generation (RAG) policy~\cite{salemi-etal-2024-lamp} that conditions on user-history examples retrieved per query.
To probe whether TTP gains depend on the choice of generator, we additionally evaluate two alternatives in Appendix~\ref{app:e2_full_table}: \emph{Persona prompting}~\cite{ryan-etal-2025-synthesizeme} (a user-style description synthesised from $15$ history samples by an external LLM) and \emph{Persona+RAG} (concatenating both).
All three use \texttt{Qwen3-4B-Instruct}~\cite{yang2025qwen3} as the generator and produce $N{=}30$ candidates per query at temperature $T{=}1.7$ (LaMP) or $T{=}1.5$ (LongLaMP).
We use $80\%$ of each user's data to train reward models and $20\%$ for evaluation.

\paragraph{Reward Models.}
We compare five selection strategies:
\textbf{Random} (uniform pick),
\textbf{Global RM} (single model on pooled data),
\textbf{Deterministic User RM} (per-user, MSE loss with high-score contrastive loss),
\textbf{Probabilistic User RM} (per-user, Gaussian NLL with high-score contrastive loss; ours), and
\textbf{Oracle} (selection by ground-truth ROUGE).
All reward models share the \texttt{Qwen2.5-1.5B-Instruct}~\cite{qwen2025qwen25} backbone with LoRA ($r{=}8$); the probabilistic variant adds a softplus variance head.
For comparison with training-based personalization, we also fine-tune per-user LoRA adapters on the policy following~\cite{tan-etal-2024-democratizing}.
Variance-based inference strategies are explored in Appendix~\ref{app:variance_selection} but offer no improvement over mean-based selection.

\paragraph{Evaluation.}
We sample $N \in \{1, 5, 10, 15, 20, 30\}$ candidates per validation query and report the average of ROUGE-1 and ROUGE-L across users.
Each experiment is repeated $10$ times; standard deviations remain below $0.007$ across methods and $N$.
Robustness to ROUGE is assessed with BERTScore and an LLM-as-judge in Appendix~\ref{app:alternative_metrics}; sensitivity to per-user data volume, LoRA rank, and statistical stability across repetitions are reported in Appendix~\ref{app:sensitivity}.

\subsection{Main Results}
\label{sec:experiments_main}

Figure~\ref{fig:main_results} presents three views of evaluation with RAG policy: scaling curves with theory overlay, per-user correlation distributions, and per-query correlation distributions.
We highlight three findings.

\paragraph{Probabilistic RM is the only method that scales reliably.}
Across all five tasks (Figure~\ref{fig:main_results}, top row), Probabilistic User RM is the only method whose Best-of-$N$ curve increases monotonically.
By $N{=}30$, Probabilistic RM closes between $51\%$ and $84\%$ of the oracle gap depending on the task, and consistently surpasses the per-user LoRA training-based baseline by $N{=}5$--$10$ on three of the five tasks (per-task numbers in Appendix~\ref{app:e2_full_table}).
The same selector ranking holds under two additional policy families (Persona prompting and Persona+RAG, Appendix~\ref{app:e2_full_table}), confirming that TTP behaves as a policy-orthogonal amplifier whose effect is determined by the reward model rather than the underlying generator.

\paragraph{Theory predictions match observations and the correlation distributions show why.}
The dash-dot theory curves in the top row are computed by substituting each reward model's four measured quantities $(\alpha, \beta, \bar{\rho}_{+}, \bar{\rho}_{-})$ into Proposition~\ref{prop:unified_law}, with no per-method fitting; they yield effective correlations $\rho_{\text{eff}} = 0.20 / 0.19 / 0.61$ for Global / Det / Prob RM and predict the observed curves at relMAE $4.4\% / 2.7\% / 1.7\%$ averaged across the five tasks (per-task table in Appendix~\ref{app:scaling_law_validation}).
The middle and bottom rows of Figure~\ref{fig:main_results} explain why these $\rho_{\text{eff}}$ values arise: probabilistic modeling shifts both per-user and per-query correlation distributions rightward and tightens them, dropping the user-level collapse rate $\alpha$ from $0.15$ (Det) to $0$ and the query-level hacking rate $\beta$ from $0.25$ to $0.06$ averaged across all five tasks.
These distributional shifts are exactly the conditions Proposition~\ref{prop:unified_law} requires for monotonic improvement, completing a tight loop between theory, diagnosis, and method.

\paragraph{Task characteristics shape TTP headroom.}
Long-form tasks (Abstract, Product, Topic) yield larger oracle gaps and correspondingly larger TTP gains than short-form tasks (Headlines, Titles), where small token differences produce large ROUGE swings and reduce learnable signal.
The diagnostic framework explains this directly: short-form tasks exhibit smaller Oracle $\bar{\sigma}$, hence a smaller logarithmic-growth coefficient $\bar{\sigma} c \sqrt{\ln N}$ (Theorem~\ref{thm:oracle_scaling}).

\subsection{Failure-Mode Analysis}

To understand \emph{when} and \emph{how} the two failure modes manifest in practice, we conduct a deeper analysis on LaMP-4 (Figure~\ref{fig:failure_mode_lamp4}); cross-task results appear in Appendix~\ref{app:failure_mode_multitask}.

\begin{figure}[t]
    \centering
    \includegraphics[width=\textwidth]{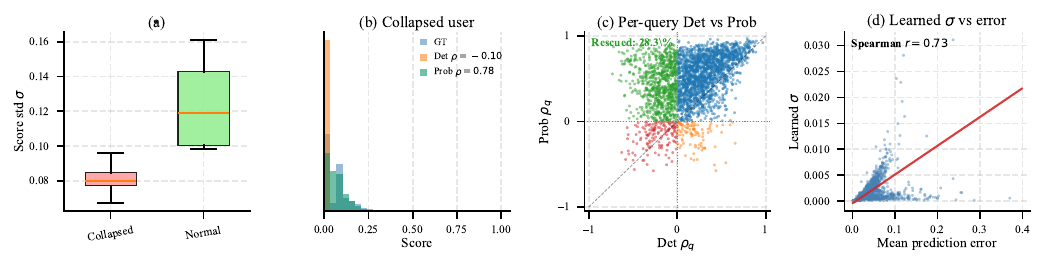}
    \caption{Failure-mode analysis on LaMP-4. \textbf{(a)} Ground-truth ROUGE standard deviation distinguishes collapsed users ($\rho_u\!<\!0.1$ under Det RM) from normal users ($\rho_u\!>\!0.5$). \textbf{(b)} Score histogram of a representative collapsed user: GT scores cluster near zero, Det RM degenerates to near-constant predictions ($\rho\!=\!{-}0.10$), while Prob RM preserves meaningful variation ($\rho\!=\!0.78$). \textbf{(c)} Per-query correlation scatter (Det vs Prob): green points in the upper-left quadrant are queries where Prob ``rescues'' Det's negative correlation. \textbf{(d)} Learned variance $\sigma^{2}$ correlates strongly with prediction error (Spearman $r\!=\!0.73$), evidence that Prob RM is well-calibrated.}
    \label{fig:failure_mode_lamp4}
\end{figure}

\paragraph{When does user-level collapse occur?}
Figure~\ref{fig:failure_mode_lamp4}~(a) groups users by their ground-truth ROUGE variance: collapsed users (under Det RM) have significantly lower label variance than normal users ($p\!<\!0.001$).
A representative collapsed user (Figure~\ref{fig:failure_mode_lamp4}~(b)) illustrates the mechanism: when GT scores cluster narrowly, Det RM emits near-constant predictions ($\rho\!=\!{-}0.10$), while Prob RM preserves meaningful variation ($\rho\!=\!0.78$).
This aligns with Lemma~\ref{lem:gradient_buffering}: NLL training increases $\sigma^{2}$ on weak-signal users, attenuating the gradient on $\mu$ and preventing the regression-to-the-mean that traps deterministic models.
The collapse rate drops from $\alpha^{\text{det}}=0.20$ to $\alpha^{\text{prob}}=0$ on LaMP-4, and to $0$ on every other task as well.

\paragraph{When does query-level hacking occur?}
Figure~\ref{fig:failure_mode_lamp4}~(c) plots per-query Det vs Prob correlations. The green points (upper-left quadrant) are queries where Det produces negative correlation but Prob produces positive correlation, accounting for $28\%$ of all LaMP-4 queries; combined with cases where both are positive (blue), Prob's negative-correlation rate drops from $\beta^{\text{det}}=0.33$ to $\beta^{\text{prob}}=0.08$.
Figure~\ref{fig:failure_mode_lamp4}~(d) explains why: Prob RM's learned $\sigma^{2}$ rises with prediction error (Spearman $r\!=\!0.73$), so during training the loss naturally down-weights queries it cannot fit.
This implicit regularization (Lemma~\ref{lem:implicit_regularization}) prevents Det RM's overfitting to spurious patterns on hard queries, which is what produces negative correlation in the first place.

%---- END sections/6_experiments.tex ----

%---- BEGIN sections/7_related_work.tex ----
\section{Related Work}
\label{sec:related}

\textbf{LLM Personalization.}
Existing approaches to LLM personalization fall into three families.
\emph{Prompting-based methods} augment the input context with user information: LaMP~\citep{salemi-etal-2024-lamp} retrieves user history via dense retrieval, while PAG~\citep{richardson2023integrating} first generates user profiles with LLMs to enrich personalization.
\emph{Adaptation-based methods} fine-tune model parameters per user: OPPU~\citep{tan-etal-2024-democratizing} introduces per-user LoRA adapters, and FDLoRA~\citep{qi2024fdlora} improves efficiency through federated learning.
\emph{Alignment-based methods} formulate personalization as multi-objective optimization: Rewarded Soup~\citep{rame2023rewarded} interpolates separately trained policy weights to balance preference dimensions.
All three modify what the model sees or how it is trained; we instead scale the policy output space at inference time and study when selection-based scaling is reliable.

\textbf{Uncertainty in Reward Modeling.}
Uncertainty quantification has been used to mitigate reward hacking in general alignment.
Ensemble approaches train multiple reward models and use disagreement as the uncertainty signal~\citep{lou2025uncertainty}.
UP-RLHF~\citep{ZHAI2026104548} trains diverse LoRA ensembles with uncertainty-penalized optimization, while PURM~\citep{sun2025probabilistic} extends Bradley--Terry to a distributional reward.
A common feature of these works is that uncertainty is consumed at \emph{inference time}, e.g., as a PPO penalty or a selection adjustment.
Our use of the variance head is structurally different: $\sigma^{2}$ acts purely at \emph{training time} via gradient buffering and implicit regularization (Lemmas~\ref{lem:gradient_buffering}, \ref{lem:implicit_regularization}); inference uses only the predicted mean.
This shift exposes a personalization-specific role for probabilistic modeling that prior alignment-focused work does not address.

\textbf{Test-Time Scaling for Personalization.}
Recent work has shown that scaling test-time compute can improve LLM alignment, but the available compute can be allocated to several distinct targets.
Three are particularly relevant for personalization:
\emph{(i) scaling user interactions}: T-POP~\citep{qu2025tpop} learns per-user preferences online via dueling bandits, requiring repeated preference queries;
\emph{(ii) scaling reward-model reasoning}: generative reward models unroll chain-of-thought tokens inside the reward model itself, trading inference cost for ranking accuracy;
and \emph{(iii) scaling policy outputs}: sampling $N$ candidates from a personalized policy and selecting the best with a per-user reward model, the axis we study here.
The three axes are complementary rather than competing; our diagnostic framework (Section~\ref{sec:theory_2}) applies to any reward model used in axis (iii) and could be combined with (i) or (ii).
Training-free LLM-as-judge methods, such as persona-conditioned pairwise tournaments, instantiate axis (ii) with a general-purpose LLM as the personalized reward model, but achieve only coin-flip accuracy on per-user style matching in our setting (Appendix~\ref{app:synthesizeme}), confirming that per-user trained reward models remain necessary for stylistic personalization.
TPO~\citep{li2025test} iteratively refines outputs through textual critiques but does not target personalization.

%---- END sections/7_related_work.tex ----

%---- BEGIN sections/8_dicussion.tex ----
\section{Conclusion and Discussion}

\textbf{Summary.}
We introduced \textbf{Test-Time Personalization} (TTP), scaling policy outputs at inference for personalized text generation.
Our central contribution is a \emph{predictive scaling law} that links four measurable reward-model properties to the Best-of-$N$ utility curve, validated empirically at relMAE $<$ 3\%.
The law identifies two failure modes (user-level collapse and query-level reward hacking), and we prove that probabilistic reward modeling overcomes both via training-time gradient buffering and implicit regularization.
The resulting probabilistic User RM is the only method whose Best-of-$N$ curve scales monotonically across all $15$ (task, policy) cells we tested, confirming TTP as a policy-orthogonal amplifier.

\textbf{Limitations.}
Our evaluation focuses on \emph{personalized text generation} with publicly available benchmarks (LaMP, LongLaMP) using ROUGE-derived ground-truth rewards.
While we show robustness to alternative metrics (BERTScore and an LLM-as-judge, Appendix~\ref{app:alternative_metrics}), generalization beyond text generation (e.g., dialogue, agentic task) remains untested.
The framework also assumes per-user data sufficient to train a small probabilistic reward model; cold-start scenarios fall outside the regime we study and are better addressed by orthogonal axes such as scaling user interactions (Section~\ref{sec:related}).
Finally, training separate reward models per user may pose deployment challenges at scale, motivating future work on parameter-efficient cross-user sharing.

\textbf{Future Directions.}
The diagnostic framework opens several avenues.
First, the three test-time axes (user interactions, reward-model reasoning, policy outputs) are complementary; combining probabilistic per-user RMs with generative reward modeling is a natural next step.
Second, extending probabilistic reward modeling to non-text modalities would test whether the gradient-buffering and implicit-regularization mechanisms transfer to settings with richer label structure.
%---- END sections/8_dicussion.tex ----

%%%%%%%%%%%%%%%%%%%%%%%%%%%%%%%%
% ACKNOWLEDGMENTS (hidden in submission via the ack environment)
%%%%%%%%%%%%%%%%%%%%%%%%%%%%%%%%
\begin{ack}
% Funding and competing-interest disclosure goes here for the
% camera-ready version. Hidden during anonymous submission.
\end{ack}

%%%%%%%%%%%%%%%%%%%%%%%%%%%%%%%%
% REFERENCES
%%%%%%%%%%%%%%%%%%%%%%%%%%%%%%%%
\bibliographystyle{unsrtnat}
\bibliography{example_paper}

%%%%%%%%%%%%%%%%%%%%%%%%%%%%%%%%
% APPENDIX
%%%%%%%%%%%%%%%%%%%%%%%%%%%%%%%%
\newpage
\appendix

\setcounter{table}{0}
\renewcommand{\thetable}{A\arabic{table}}
\setcounter{figure}{0}
\renewcommand{\thefigure}{A\arabic{figure}}

%---- BEGIN sections/appendix.tex ----
\section{Theoretical Proofs}
\label{app:proofs}

In this section, we provide formal proofs or theoretical analysis for all theoretical results presented in the main text. 
Our theoretical framework rests on a chain of results that progressively characterizes test-time personalization:

\begin{enumerate}
    \item \textbf{Theorem~\ref{thm:oracle_scaling} (Oracle Scaling Law, Section~\ref{app:proof_oracle})}: Establishes the theoretical ceiling for TTP, expected utility grows as $O(\sqrt{\ln N})$ with oracle selection.
    
    \item \textbf{Lemma~\ref{lem:corr_scaling} (Correlation-Scaling Relationship, Section~\ref{app:proof_correlation})}: Shows that reward model correlation directly determines scaling behavior, providing a diagnostic tool for analyzing RM quality.
    
    \item \textbf{Proposition~\ref{prop:unified_law} (Unified Scaling Law, Section~\ref{app:unified_law})}: Derives how the two failure modes, collapse rate $\alpha$ and hacking rate $\beta$, jointly determine population-level scaling.
    
    \item \textbf{Lemmas~\ref{lem:gradient_buffering} and \ref{lem:implicit_regularization} (Gradient Buffering \& Implicit Regularization, Section~\ref{app:proof_lemma51_52})}: Explains the mechanisms by which probabilistic reward modeling reduces both failure modes.
\end{enumerate}

Throughout the proofs, we introduce necessary assumptions and provide remarks connecting theoretical insights to empirical observations. Table~\ref{tab:assumptions_summary} summarizes the key assumptions used in our analysis.

\begin{table}[h]
\centering
\caption{Summary of assumptions used in theoretical analysis.}
\label{tab:assumptions_summary}
\begin{tabular}{lll}
\toprule
\textbf{Assumption} & \textbf{Used In} & \textbf{Description} \\
\midrule
Sub-Gaussian rewards & Theorem~\ref{thm:oracle_scaling} & Reward distributions have light tails \\
Correlation uniformity & Lemma~\ref{lem:corr_scaling} & Local correlation $\approx$ global correlation \\
Collapse absorbs hacking & Proposition~\ref{prop:unified_law} & $\beta$ measured on non-collapsed users \\
Variance homogeneity & Proposition~\ref{prop:unified_law} & $\sigma$ variation is second-order \\
\bottomrule
\end{tabular}
\end{table}

%==============================================================================

%==============================================================================

\subsection{Proof of Theorem 3.1}
\label{app:proof_oracle}

We first introduce the sub-Gaussian property and a key lemma on the maximum of random variables, then prove Theorem 3.1.

\begin{definition}[Sub-Gaussian Random Variable]
\label{def:subgaussian}
A random variable $X$ with mean $\mu = \mathbb{E}[X]$ is $\sigma$-sub-Gaussian if:
\begin{equation}
    \mathbb{E}\left[\exp\left(\lambda(X - \mu)\right)\right] \leq \exp\left(\frac{\lambda^2 \sigma^2}{2}\right), \quad \forall \lambda \in \mathbb{R}
\end{equation}
where $\sigma > 0$ is the sub-Gaussian parameter. This assumption covers a wide range of bounded or light-tailed distributions, including typical LLM quality scores.
\end{definition}

\begin{lemma}[Upper Bound on Expected Maximum]
\label{lemma:max_subgaussian}
Let $X_1, \ldots, X_N$ be i.i.d.\ $\sigma$-sub-Gaussian random variables with mean $\mu$. Then the expected maximum is bounded by:
\begin{equation}
    \mathbb{E}\left[\max_{i \in [N]} X_i\right] \leq \mu + \sigma\sqrt{2 \ln N}
\end{equation}
\end{lemma}

\begin{proof}
For any $\lambda > 0$, using Jensen's inequality and the monotonicity of the exponential function:
\begin{equation}
    \exp\left(\lambda \mathbb{E}\left[\max_i (X_i - \mu)\right]\right) \leq \mathbb{E}\left[\exp\left(\lambda \max_i (X_i - \mu)\right)\right] = \mathbb{E}\left[\max_i \exp(\lambda (X_i - \mu))\right]
\end{equation}
Bounding the maximum by the sum:
\begin{equation}
    \mathbb{E}\left[\max_i \exp(\lambda (X_i - \mu))\right] \leq \sum_{i=1}^{N} \mathbb{E}\left[\exp(\lambda (X_i - \mu))\right]
\end{equation}
Applying the sub-Gaussian property (Definition~\ref{def:subgaussian}):
\begin{equation}
    \sum_{i=1}^{N} \mathbb{E}\left[\exp(\lambda (X_i - \mu))\right] \leq N \exp\left(\frac{\lambda^2 \sigma^2}{2}\right)
\end{equation}
Taking logarithms on both sides:
\begin{equation}
    \lambda \mathbb{E}\left[\max_i (X_i - \mu)\right] \leq \ln N + \frac{\lambda^2 \sigma^2}{2} \implies \mathbb{E}\left[\max_i X_i\right] \leq \mu + \frac{\ln N}{\lambda} + \frac{\lambda \sigma^2}{2}
\end{equation}
The right-hand side is minimized when $\lambda = \sqrt{2\ln N}/\sigma$. Substituting this value yields the bound $\mu + \sigma\sqrt{2 \ln N}$.
\end{proof}

\begin{proof}
For a user $u$ with query $q$, assuming the true rewards $\{r^*_u(q, x_i)\}_{i=1}^N$ are i.i.d.\ $\sigma_u$-sub-Gaussian with mean $\mu_u$, we apply Lemma~\ref{lemma:max_subgaussian}:
\begin{equation}
    U_{\text{oracle},u}(N) = \mathbb{E}\left[\max_{i \in [N]} r^*_u(q, x_i)\right] \leq \mu_u + c \cdot \sigma_u \cdot \sqrt{\ln N}
\end{equation}
where $c = \sqrt{2}$ is the theoretical scaling coefficient derived from the sub-Gaussian bound. 

Taking the expectation over the population of users:
\begin{equation}
    \bar{U}_{\text{oracle}}(N) = \mathbb{E}_u[U_{\text{oracle},u}(N)] \leq \mathbb{E}_u[\mu_u] + c\sqrt{\ln N} \cdot \mathbb{E}_u[\sigma_u] = \bar{\mu} + c \bar{\sigma} \sqrt{\ln N}
\end{equation}
Thus, the oracle utility is upper bounded by a logarithmic growth term scaled by the average reward noise $\bar{\sigma}$.
\end{proof}

\paragraph{Remark on Bound Tightness.} 
While Theorem~\ref{thm:oracle_scaling} provides an upper bound with $c=\sqrt{2}$, our empirical observations suggest that the scaling curve tightly follows this boundary. This indicates that the reward distributions in practice are not merely sub-Gaussian but exhibit tail behaviors close to the Gaussian limit (where the bound becomes an asymptotic equality). Consequently, we treat $c$ as an effective constant in our analysis to absorb minor distributional deviations, recognizing that $c \approx \sqrt{2}$ represents the fully saturated theoretical potential.

%==============================================================================

%==============================================================================

\subsection{Proof of Lemma 4.2}
\label{app:proof_correlation}

This lemma establishes how reward model correlation determines TTP scaling behavior. We first present the idealized result under a uniformity assumption, then discuss when this assumption breaks down in practice.

\paragraph{Setup and Notation}
For a user $u$, let $r^* = r^*_u(q, x)$ denote the true reward and $\hat{r} = \hat{r}_u(q, x)$ denote the predicted reward for a query-response pair. We assume both are standardized to have mean $\mu$ and consider their correlation:
\begin{equation}
    \rho_u = \text{Corr}(\hat{r}, r^*) = \frac{\mathbb{E}[(\hat{r} - \mu)( r^* - \mu)]}{\sigma_{\hat{r}} \sigma_{r^*}}
\end{equation}

\begin{assumption}[Correlation Uniformity]
\label{assume:uniformity}
The correlation $\rho_u$ between predicted and true rewards is approximately uniform across the score distribution. Formally, for any subset $\mathcal{S} \subseteq [0, 1]$ of the score range:
\begin{equation}
    \rho_u^{(\mathcal{S})} := \text{Corr}(\hat{r}, r^* \mid r^* \in \mathcal{S}) \approx \rho_u
\end{equation}
\end{assumption}

This assumption is analogous to homoscedasticity in regression: the predictive relationship is stable across different regions of the target distribution.

\begin{lemma}[Correlation-Scaling Relationship, Restated]
Under Assumption~\ref{assume:uniformity}, for a reward model with correlation $\rho_u$, the Best-of-N utility satisfies:
\begin{equation}
    U_u(N) \approx \mu_u + \rho_u \cdot \sigma_u \cdot c\sqrt{\ln N}
\end{equation}
where $c > 0$ is the same constant as in Theorem~\ref{thm:oracle_scaling}.
\end{lemma}

\begin{proof}
We assume the joint distribution of predicted and true rewards follows a Bivariate Normal distribution (or satisfies the property of linear conditional expectation). 
Under joint normality (or linear regression structure), the conditional expectation of true reward given predicted reward is:
\begin{equation}
    \mathbb{E}[r^* \mid \hat{r}] = \mu + \rho_u \cdot \frac{\sigma_{r^*}}{\sigma_{\hat{r}}} (\hat{r} - \mu)
\end{equation}

Let $\hat{x}^* = \arg\max_{i \in [N]} \hat{r}(x_i)$ be the candidate selected by the reward model. The expected true reward of this selection is:
\begin{equation}
    \mathbb{E}[r^*(\hat{x}^*)] = \mathbb{E}\left[\mathbb{E}[r^* \mid \hat{r}(\hat{x}^*)]\right] = \mu + \rho_u \cdot \frac{\sigma_{r^*}}{\sigma_{\hat{r}}} \cdot \mathbb{E}[\hat{r}(\hat{x}^*) - \mu]
\end{equation}

Since $\hat{r}(x_i)$ are i.i.d.\ sub-Gaussian, by Lemma~\ref{lemma:max_subgaussian}:
\begin{equation}
    \mathbb{E}[\hat{r}(\hat{x}^*)] = \mathbb{E}\left[\max_{i \in [N]} \hat{r}(x_i)\right] = \mu + c \cdot \sigma_{\hat{r}} \cdot \sqrt{\ln N}
\end{equation}

Substituting into the previous equation:
\begin{equation}
    U_u(N) = \mathbb{E}[r^*(\hat{x}^*)] = \mu + \rho_u \cdot \frac{\sigma_{r^*}}{\sigma_{\hat{r}}} \cdot c \cdot \sigma_{\hat{r}} \cdot \sqrt{\ln N} = \mu_u + \rho_u \cdot \sigma_u \cdot c\sqrt{\ln N}
\end{equation}
where $\sigma_u := \sigma_{r^*}$ for notational consistency.
\end{proof}

\paragraph{Remark.} 
In practice, Assumption~\ref{assume:uniformity} often does not hold exactly. High-score regions are typically harder to model due to data sparsity, subtle quality distinctions, and regression to the mean, leading to $\rho_u^{(\text{high})} < \rho_u^{(\text{overall})}$. Since Best-of-N selection operates primarily in the right tail of the score distribution, the effective correlation for TTP is lower than the overall correlation. This explains why in our main experiments (Figure~\ref{fig:main_results}), even when probabilistic RM achieves high user-level correlation on LongLaMP tasks, a gap remains between TTP performance and the oracle upper bound.

%==============================================================================

%==============================================================================

\subsection{Proof of Proposition 4.4}
\label{app:unified_law}

This proposition establishes how the two failure modes, user-level collapse and query-level reward hacking, jointly determine TTP scaling behavior.

\paragraph{Definitions and Assumptions.}

We first formalize the failure rates introduced in Definition~\ref{def:failure_rates}:

\begin{definition}[Failure Rates, Restated]
For a reward modeling approach:
\begin{itemize}
    \item \textbf{Collapse rate} $\alpha$: Fraction of users with per-user correlation $\rho_u < \tau_c$, where $\tau_c = 0.1$ is the collapse threshold.
    \item \textbf{Reward hacking rate} $\beta$: Fraction of queries (among non-collapsed users) with per-query correlation $\rho_q < 0$.
\end{itemize}
\end{definition}

\begin{assumption}[\textbf{Collapse Absorbs Hacking}]
\label{assume:collapse_absorbs}
For collapsed users ($\rho_u < \tau_c$), the reward model outputs near-constant predictions, rendering query-level correlation undefined or negligible. Thus, the hacking rate $\beta$ is measured only among non-collapsed users.
\end{assumption}
This assumption reflects the intuition that collapse represents a more severe failure mode: when a reward model cannot discriminate between candidates at all, the notion of ``selecting worse candidates'' (hacking) becomes moot.

\begin{assumption}[\textbf{Variance Homogeneity}]
\label{assume:variance_homogeneity}
The variance $\sigma_u$ of true rewards is approximately constant across users and query types. Variations in $\sigma$ across collapsed/non-collapsed users or hacked/non-hacked queries are treated as second-order effects.
\end{assumption}

\begin{proposition}[\textbf{Unified Scaling Law}, Restated]
\label{prop:unified_scaling}
Under Assumptions~\ref{assume:uniformity}, \ref{assume:collapse_absorbs}, and \ref{assume:variance_homogeneity}, the population-level Best-of-N utility can be approximated as:
\begin{equation}
    \bar{U}(N) \approx \bar{\mu} + (1 - \alpha) \cdot \left[(1 - \beta)\bar{\rho}_+ - \beta|\bar{\rho}_-|\right] \cdot \bar{\sigma} \cdot c\sqrt{\ln N}
\end{equation}
where $\bar{\rho}_+ > 0$ is the mean correlation on non-hacked queries and $\bar{\rho}_- < 0$ is the mean correlation on hacked queries (both measured among non-collapsed users).
\end{proposition}

\begin{proof}
We derive the result through a two-level decomposition. Partition users into collapsed (fraction $\alpha$) and non-collapsed (fraction $1-\alpha$):
\begin{equation}
    \bar{U}(N) = \alpha \cdot \bar{U}_{\text{collapsed}}(N) + (1 - \alpha) \cdot \bar{U}_{\text{non-collapsed}}(N)
\end{equation}

For collapsed users, $\rho_u \approx 0$, so by Lemma~\ref{lem:corr_scaling}, $\bar{U}_{\text{collapsed}}(N) \approx \bar{\mu}$.

Among non-collapsed users, partition queries into hacked (fraction $\beta$) and non-hacked (fraction $1-\beta$). Applying Lemma~\ref{lem:corr_scaling} to each partition:
\begin{align}
    \bar{U}_{\text{non-hacked}}(N) &\approx \bar{\mu} + \bar{\rho}_+ \cdot \bar{\sigma} \cdot c\sqrt{\ln N} \\
    \bar{U}_{\text{hacked}}(N) &\approx \bar{\mu} + \bar{\rho}_- \cdot \bar{\sigma} \cdot c\sqrt{\ln N}
\end{align}

Combining the query-level results:
\begin{equation}
    \bar{U}_{\text{non-collapsed}}(N) = \bar{\mu} + \left[(1 - \beta)\bar{\rho}_+ - \beta|\bar{\rho}_-|\right] \cdot \bar{\sigma} \cdot c\sqrt{\ln N}
\end{equation}

Substituting back into the user-level decomposition:
\begin{equation}
    \bar{U}(N) = \bar{\mu} + (1 - \alpha) \cdot \left[(1 - \beta)\bar{\rho}_+ - \beta|\bar{\rho}_-|\right] \cdot \bar{\sigma} \cdot c\sqrt{\ln N}
\end{equation}
\end{proof}

\paragraph{Remark on the trade-off.}
Proposition~\ref{prop:unified_scaling} reveals a fundamental trade-off: Global RM achieves stability (low $\alpha$, low $\beta$) but weak personalization (low $\bar{\rho}_+$); deterministic User RM achieves high $\bar{\rho}_+$ but suffers from instability (high $\alpha$ or $\beta$); probabilistic User RM achieves both high correlation and stability through gradient buffering and implicit regularization (Lemmas~\ref{lem:gradient_buffering} and \ref{lem:implicit_regularization}), enabling effective TTP. The unified scaling law thus provides a diagnostic framework for predicting TTP performance by measuring $\alpha$, $\beta$, $\bar{\rho}_+$, and $\bar{\rho}_-$.

\paragraph{Remark on the collapsed-user assumption.}
Assumption~\ref{assume:collapse_absorbs} treats the contribution of collapsed users as negligible because their per-user correlation is approximately zero. In our data, however, collapsed users tend to have a small but non-zero negative correlation (e.g., $\bar{\rho}_{\text{collapsed}} \approx -0.027$ on LaMP-4 and $-0.030$ on LongLaMP-Product Review). Replacing the $\rho_u \approx 0$ assumption with the empirically measured $\bar{\rho}_{\text{collapsed}}$ yields a refined formula
\begin{equation}
    \bar{U}(N) \approx \bar{\mu} + \big[(1-\alpha)\rho_{\text{eff}} + \alpha \bar{\rho}_{\text{collapsed}}\big] \cdot \bar{\sigma} \cdot c\sqrt{\ln N}.
\end{equation}
The refinement is small in magnitude (improving relMAE on LaMP-4 Det RM from $9.1\%$ to $8.8\%$, similarly on the other tasks where collapsed users are most prevalent) and \emph{strengthens} the main conclusion: collapsed users contribute slightly \emph{negative} scaling, so probabilistic modeling, which eliminates collapse altogether ($\alpha^{\text{prob}}{=}0$), is even more important than the original analysis suggests.

%==============================================================================

%==============================================================================

\subsection{Theoretical Analysis of Lemma 5.1 and Lemma 5.2}
\label{app:proof_lemma51_52}

In this section, we provide theoretical analysis of how probabilistic reward modeling addresses the two failure modes identified in Section~\ref{sec:theory_2}. 
Both mechanisms operate during training by leveraging the learned variance $\sigma^2$ to modulate gradient flow and sample weighting.

\paragraph{Preliminaries: NLL vs MSE Training.}

Consider a reward model that predicts mean $\mu(x)$ and variance $\sigma^2(x)$ for input $x$ with ground-truth label $y$. The two training objectives are:

\textbf{MSE Loss:}
\begin{equation}
    \mathcal{L}_{\text{MSE}} = (y - \mu)^2
\end{equation}

\textbf{Gaussian NLL Loss:}
\begin{equation}
    \mathcal{L}_{\text{NLL}} = \frac{1}{2}\log \sigma^2 + \frac{(y - \mu)^2}{2\sigma^2}
\end{equation}

The key difference lies in their gradients with respect to $\mu$:
\begin{align}
    \frac{\partial \mathcal{L}_{\text{MSE}}}{\partial \mu} &= -2(y - \mu) \\
    \frac{\partial \mathcal{L}_{\text{NLL}}}{\partial \mu} &= \frac{\mu - y}{\sigma^2}
\end{align}

Under NLL training, the gradient magnitude is scaled by $1/\sigma^2$, allowing the model to modulate learning signals through the variance prediction.

\paragraph{Theoretical Analysis of Lemma 5.1.}
\label{app:collapse_analysis}

\begin{lemma}[\textbf{Gradient Buffering}, Restated]
\label{lem:gradient_buffering_proof}
Under NLL training, the effective gradient on $\mu$ scales as $O(1/\sigma^2)$. For users with narrow score distributions where training signals provide weak discriminative information, the model learns to increase $\sigma^2$, thereby attenuating gradient magnitudes and preventing the collapse that occurs in deterministic models trained with MSE.
\end{lemma}

\begin{proof}
Consider a user $u$ whose ground-truth reward distribution has low variance (narrow distribution with small $\Delta_y := y_{\max} - y_{\min}$).

Under MSE training, when $\Delta_y$ is small and inputs $x_i$ are diverse, the model faces conflicting gradients: different inputs with similar labels pull $\mu$ toward the same value. The equilibrium is a near-constant prediction $\mu_{\text{MSE}}(x) \approx \bar{y}$, constituting user-level collapse.

Under NLL training, the model jointly optimizes $\mu$ and $\sigma^2$. The gradient with respect to $\sigma^2$ is:
\begin{equation}
    \frac{\partial \mathcal{L}_{\text{NLL}}}{\partial \sigma^2} = \frac{1}{2\sigma^2} - \frac{(y - \mu)^2}{2\sigma^4}
\end{equation}

Setting this to zero yields the per-sample equilibrium $\sigma^{*2} = (y - \mu)^2$. When the model cannot accurately predict $\mu$, the residual $(y - \mu)^2$ remains large, and $\sigma^2$ increases to match this residual. The gradient on $\mu$ then becomes:
\begin{equation}
    \frac{\partial \mathcal{L}_{\text{NLL}}}{\partial \mu} = \frac{\mu - y}{\sigma^2} = O(1/|y - \mu|)
\end{equation}

As $\sigma^2$ increases, gradient magnitudes on $\mu$ decrease, preventing the model from being forced to fit contradictory signals and allowing $\mu$ to maintain meaningful variation.
\end{proof}

\paragraph{Theoretical Analysis of Lemma 5.2.}
\label{app:hacking_analysis}

\begin{lemma}[\textbf{Implicit Regularization}, Restated]
\label{lem:implicit_reg_proof}
Under NLL training, the learned variance $\sigma^2$ acts as an adaptive sample weight, effectively down-weighting samples with unreliable or inconsistent labels. This implicit regularization prevents overfitting to spurious patterns, resulting in mean predictions $\mu$ that exhibit higher correlation with true rewards and reduced reward hacking.
\end{lemma}

\begin{proof}[Justification]
The NLL loss can be decomposed as:
\begin{equation}
    \mathcal{L}_{\text{NLL}} = \frac{1}{2}\log \sigma^2 + \frac{1}{2\sigma^2}(y - \mu)^2
\end{equation}

The second term is a weighted MSE with sample-dependent weight $w = 1/\sigma^2$. For samples where the input-output relationship is inconsistent or noisy, the model cannot achieve low residuals through $\mu$ alone. The optimal response is to increase $\sigma^2$ for these samples, which reduces their contribution to the MSE term while incurring a modest penalty from the $\log \sigma^2$ term (preventing $\sigma^2 \to \infty$).

Spurious patterns arise when the model memorizes incidental correlations that do not generalize. Under NLL training, precisely these noisy samples receive high $\sigma^2$ predictions, reducing their influence on $\mu$. This prevents the model from confidently learning spurious correlations. By down-weighting noisy samples during training, the resulting $\mu$ predictions are more aligned with genuine quality signals, yielding higher correlation $\rho_q$ and lower hacking rate $\beta$.
\end{proof}

\paragraph{Remark.}
A key insight is that the variance head $\sigma^2$ functions as an auxiliary task that improves learning of $\mu$ without directly participating in inference-time decisions. Both gradient buffering and implicit regularization operate entirely during training. At inference time, we select candidates using only $\mu$, but this $\mu$ is more robust than one trained with MSE alone. Notably, using $\sigma^2$ at inference time (e.g., filtering out high-variance predictions) does not further improve performance in our experiments, reinforcing that the primary value of probabilistic modeling lies in training-time regularization.

%==============================================================================

%==============================================================================

\section{Additional Experimental Results}

\subsection{Generalization Across Policy Families}
\label{app:e2_full_table}

The main-text results use a RAG policy. To probe whether TTP gains depend on the choice of generator, we additionally evaluate two other policy families: \emph{Persona prompting} (a user-style description synthesised from $15$ history samples by an external LLM) and \emph{Persona+RAG} (concatenating both). Figure~\ref{fig:main_results_generalization} shows the corresponding scaling curves; Tables~\ref{tab:e2_persona_only} and~\ref{tab:e2_persona_rag} list the numerical Best-of-$N{=}30$ results.

\begin{figure}[h]
    \centering
    \includegraphics[width=\textwidth]{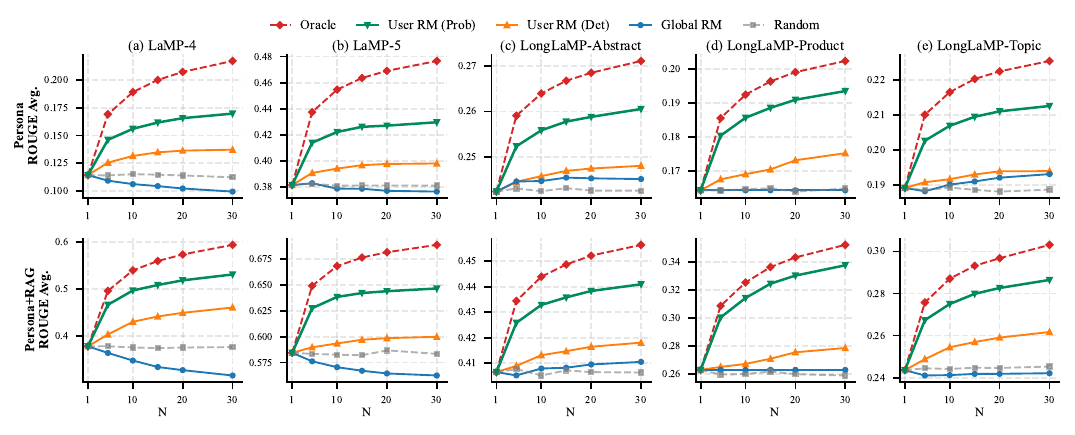}
    \caption{TTP scaling under two additional policy families across all five tasks. Top row: Persona prompting (weak $N{=}1$ baseline). Bottom row: Persona+RAG (strong $N{=}1$ baseline). The selector ranking Probabilistic $>$ Deterministic $>$ Global $\approx$ Random is preserved on every (task, policy) pair, confirming that TTP acts as a policy-orthogonal amplifier whose effect is determined by the reward model rather than the underlying generator.}
    \label{fig:main_results_generalization}
\end{figure}

``Capture'' below denotes $(\text{Prob} - \text{Random})/(\text{Oracle} - \text{Random})$, the fraction of the oracle gap recovered by the probabilistic reward model.

\begin{table}[h]
\centering
\caption{Persona prompting policy: ROUGE at $N{=}30$.}
\label{tab:e2_persona_only}
\begin{tabular}{lcccccc}
\toprule
\textbf{Task} & Random & Global & Det. & \textbf{Prob.} & Oracle & Capture \\
\midrule
LaMP-4   & 0.112 & 0.099 & 0.137 & \textbf{0.170} & 0.218 & 55\% \\
LaMP-5   & 0.381 & 0.376 & 0.398 & \textbf{0.430} & 0.477 & 51\% \\
Abstract & 0.243 & 0.245 & 0.248 & \textbf{0.261} & 0.271 & 63\% \\
Topic    & 0.189 & 0.193 & 0.194 & \textbf{0.213} & 0.226 & 65\% \\
Product  & 0.165 & 0.164 & 0.175 & \textbf{0.194} & 0.203 & 76\% \\
\bottomrule
\end{tabular}
\end{table}

\begin{table}[h]
\centering
\caption{Persona+RAG policy: ROUGE at $N{=}30$.}
\label{tab:e2_persona_rag}
\begin{tabular}{lcccccc}
\toprule
\textbf{Task} & Random & Global & Det. & \textbf{Prob.} & Oracle & Capture \\
\midrule
LaMP-4   & 0.376 & 0.316 & 0.460 & \textbf{0.531} & 0.594 & 71\% \\
LaMP-5   & 0.584 & 0.563 & 0.600 & \textbf{0.647} & 0.689 & 60\% \\
Abstract & 0.406 & 0.411 & 0.418 & \textbf{0.441} & 0.456 & 69\% \\
Topic    & 0.245 & 0.242 & 0.262 & \textbf{0.286} & 0.303 & 71\% \\
Product  & 0.259 & 0.263 & 0.278 & \textbf{0.337} & 0.352 & 84\% \\
\bottomrule
\end{tabular}
\end{table}

The probabilistic User RM is the best non-oracle selector on every (task, policy) pair.
Oracle capture exceeds $50\%$ in all $10$ cells and reaches $84\%$ on Persona+RAG Product Review.
The capture rate is highest for the strongest policy (Persona+RAG), confirming that stronger generators provide more high-quality candidates the reward model can exploit.
Crucially, even on Persona+RAG with $N{=}1$ baselines as high as $0.376$, Global RM still hacks (LaMP-4: $0.376 \to 0.316$), confirming that hacking is a property of the reward model rather than of the candidate quality range.

%==============================================================================

%==============================================================================

\paragraph{TTP with a per-user LoRA training-based policy.}
\label{app:policy_comparison}
We further evaluate TTP on top of a strong training-based personalization method: per-user LoRA fine-tuning of the policy~\cite{tan-etal-2024-democratizing}, which substantially raises the $N{=}1$ baseline.
Figure~\ref{fig:policy_comparison} compares RAG-based and LoRA-based generation on LaMP-4 under the same probabilistic User RM.
The LoRA policy reaches a higher oracle ceiling but its $N{=}1$ baseline is also much higher; the probabilistic User RM continues to scale, with reduced magnitude reflecting distribution shift between the RAG-trained RM and LoRA-generated candidates.
This is consistent with Theorem~\ref{thm:oracle_scaling}, which bounds TTP gain by $\bar{\sigma}\sqrt{\ln N}$: a stronger policy reduces $\bar{\sigma}$ and therefore the achievable headroom, but does not invalidate the scaling mechanism.

\begin{figure}[h]
    \centering
    \includegraphics[width=0.55\textwidth]{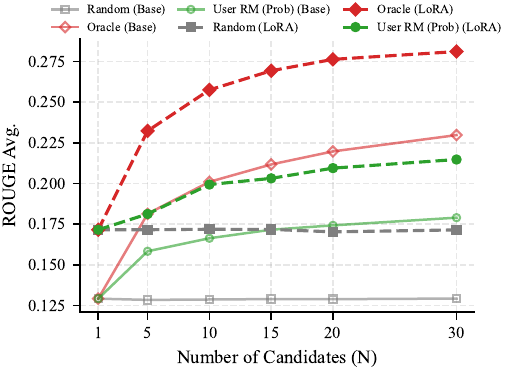}
    \caption{TTP scaling curves comparing RAG-based (Base, dashed) and LoRA-based (solid) policy models on LaMP-4. The LoRA policy reaches a higher oracle ceiling but its baseline is also much higher; the probabilistic User RM continues to scale, with reduced magnitude reflecting distribution shift between the RAG-trained RM and LoRA-generated candidates.}
    \label{fig:policy_comparison}
\end{figure}

%==============================================================================

%==============================================================================

\subsection{Robustness to Evaluation Metric}
\label{app:alternative_metrics}

ROUGE provides a convenient ground-truth proxy but may not fully capture semantic or stylistic preferences.
We replicate the central findings under two alternative metrics.

\paragraph{BERTScore correlates with ROUGE and replicates the trend.}
We compute BERTScore (\texttt{all-MiniLM-L6-v2} embeddings, cosine similarity) for all generated candidates and recompute the per-task Pearson correlation with ROUGE.
Table~\ref{tab:bertscore_correlation} shows moderate-to-high correlation across all five tasks.
On LaMP-4 at $N{=}30$, the BERTScore-rated outputs of selectors are: Random $0.372$, Global $0.348$ ($-5.7\%$), Deterministic $0.362$ ($-2.1\%$), \textbf{Probabilistic $0.392$ ($+6.0\%$)}, ROUGE-Oracle $0.410$.
The ranking Prob $>$ Random $>$ Det $>$ Global is preserved, and Global RM's negative scaling under BERTScore independently confirms reward hacking is not a ROUGE artefact.

\begin{table}[h]
\centering
\caption{BERTScore--ROUGE correlation on candidates pooled across all users and queries.}
\label{tab:bertscore_correlation}
\begin{tabular}{lcc}
\toprule
\textbf{Task} & \textbf{Pearson} & \textbf{Spearman} \\
\midrule
LaMP-4   & 0.636 & 0.611 \\
LaMP-5   & 0.664 & 0.656 \\
Abstract & 0.568 & 0.540 \\
Topic    & 0.415 & 0.406 \\
Product  & 0.448 & 0.447 \\
\bottomrule
\end{tabular}
\end{table}

\paragraph{LLM-as-judge fails for stylistic personalization.}
We additionally evaluated a strong general-purpose LLM (GPT-class, given user-history context) as a pointwise scoring judge ($1$--$10$ scale) on LaMP-4.
The Pearson correlation between judge scores and ROUGE is $r = 0.099$, essentially random.
The ground-truth answer ranks at average position $12.7/30$ (random would be $15.5$) and lands in the top-$5$ only $34.7\%$ of the time.
Under Best-of-$N{=}30$ judge selection, ROUGE is $0.144$, compared to Random $0.134$ and Probabilistic $0.179$.
General-purpose LLMs cannot capture user-level stylistic preferences from limited history, motivating per-user trained reward models.

%==============================================================================

%==============================================================================

%==============================================================================

%==============================================================================

\subsection{Ablation Studies}
\label{app:ablation}

\paragraph{Architectural components.}
We conduct ablation studies on LaMP-4 to understand the contribution of each component in our probabilistic reward model. 
Figure~\ref{fig:ablation} illustrates how each variant scales with the number of candidates $N$. The full model consistently outperforms all ablated variants across all values of $N$, with the performance gap widening as $N$ increases. Notably, the last-pooling variant shows \emph{negative} scaling---its performance decreases with larger $N$---indicating that it learns a reward function that is inversely correlated with true preferences for many queries, a manifestation of the reward hacking phenomenon discussed in Section~\ref{sec:theory_2}.

\begin{figure}[t]
    \centering
    \includegraphics[width=0.45\linewidth]{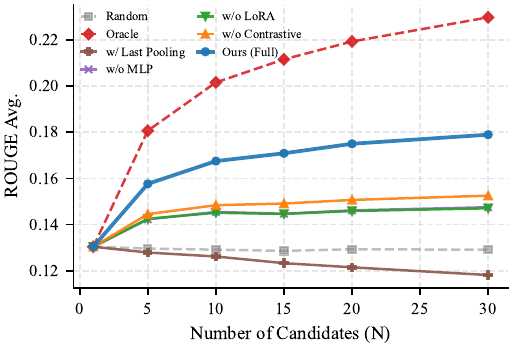}
    \caption{Ablation study: TTS performance vs number of candidates $N$ on LaMP-4. The full model (green) consistently outperforms all ablated variants. Last-token pooling (purple) performs worse than random selection, demonstrating catastrophic failure when holistic sequence understanding is lost.}
    \label{fig:ablation}
\end{figure}

\textbf{Pooling Strategy is Critical.} The most striking finding is that replacing mean pooling with last-token pooling causes catastrophic failure, with performance dropping \emph{below} the random baseline (12.46\% vs 13.58\%). This demonstrates that for reward modeling in personalization tasks, aggregating information across all tokens is essential. Last-token pooling, commonly used in decoder-only language models, fails to capture the holistic semantic relationship between query and candidate, which is crucial for accurate preference prediction.

\textbf{LoRA Adaptation Enables Personalization.} Removing the LoRA adapter results in a substantial 3.43\% drop in ROUGE-1, reducing oracle capture from 49.1\% to 17.3\%. This confirms that adapting the backbone representations to individual users is necessary for learning personalized preferences. Without LoRA, the model relies solely on the pre-trained representations, which cannot distinguish between different users' preference patterns.

\textbf{MLP Head Improves Expressiveness.} Replacing the MLP head with a simpler architecture causes a similar 3.36\% performance drop. The two-layer MLP with GELU activation provides the necessary nonlinearity to map pooled representations to accurate reward predictions. This is particularly important for modeling the complex, user-specific scoring functions required for personalization.

\textbf{Contrastive Learning Provides Additional Gains.} Removing the high-score contrastive loss results in a 2.83\% drop, the smallest among all ablations. While contrastive learning helps the model better distinguish between high-quality candidates, the primary gains come from the architectural choices (pooling, LoRA, MLP) rather than the auxiliary loss. Nevertheless, the contrastive objective provides meaningful improvements by encouraging the model to learn finer-grained preference distinctions.

%==============================================================================

%==============================================================================

\paragraph{Loss decomposition: NLL vs.\ contrastive.}
The architectural ablations in Appendix~\ref{app:ablation} fix the loss as NLL+contrastive and vary the architecture.
Here we hold the architecture fixed and vary the loss to decompose the contributions of NLL and the contrastive term.
Results on LaMP-4 are reported in Table~\ref{tab:nll_contrastive}.

\begin{table}[h]
\centering
\caption{Loss-decomposition ablation on LaMP-4 (Probabilistic User RM architecture, $r{=}8$).}
\label{tab:nll_contrastive}
\begin{tabular}{lccc}
\toprule
\textbf{Loss} & RM Pearson & ROUGE@$N{=}30$ & Gain over Random \\
\midrule
MSE only          & 0.790 & 0.139 & $+8.2\%$ \\
MSE + Contrastive & 0.790 & 0.139 & $+8.0\%$ \\
NLL only          & 0.916 & 0.156 & $+21.4\%$ \\
\textbf{NLL + Contrastive (full)} & \textbf{0.916} & \textbf{0.179} & $\mathbf{+39.3\%}$ \\
\bottomrule
\end{tabular}
\end{table}

Two findings emerge.
First, NLL is the primary contributor: replacing MSE with NLL alone closes most of the gap from $0.139$ to $0.156$ (Pearson lifts from $0.790$ to $0.916$), consistent with the prediction that NLL training prevents collapse via gradient buffering (Lemma~\ref{lem:gradient_buffering}).
Second, contrastive loss is complementary but only effective on top of NLL: the same contrastive term added to MSE produces no improvement, while added to NLL it lifts ROUGE from $0.156$ to $0.179$.
NLL provides the variance-aware foundation; contrastive refines ranking in the high-quality region critical for Best-of-$N$.

%==============================================================================

%==============================================================================

%==============================================================================

%==============================================================================

\subsection{General LLMs as Personalized Reward Models}
\label{app:synthesizeme}

To probe whether a strong general-purpose LLM with explicit user-persona context can replace per-user trained reward models, we evaluated a SynthesizeMe-style pairwise judge on LaMP-4: the judge LLM (\texttt{Qwen3-4B} with user-persona prompt) compares two candidates at a time, and a single-elimination knockout tournament selects the best of $N{=}30$ candidates per query ($29$ LLM calls per query).

\begin{table}[h]
\centering
\caption{Persona-conditioned LLM tournament vs.\ trained reward models on LaMP-4 (Best-of-$N{=}30$).}
\label{tab:synthesizeme}
\begin{tabular}{lcccc}
\toprule
\textbf{Method} & Pairwise acc.\ & ROUGE@$N{=}30$ & Gain & Cost/query \\
\midrule
Random                          & 0.500 & 0.155 & ---       & --- \\
LLM tournament (Qwen3-4B)       & \textbf{0.499} & 0.157 & $+1.3\%$ & $1.63$\,s \\
LLM pointwise (GPT-class)       & ---   & 0.144 & $-7.1\%$ & $30$ calls \\
\textbf{Probabilistic User RM}  & ---   & \textbf{0.179} & $\mathbf{+15.5\%}$ & $\mathbf{0.05}$\,s \\
\bottomrule
\end{tabular}
\end{table}

The LLM tournament's pairwise accuracy against ROUGE-determined preferences is $0.499$, indistinguishable from chance.
The corresponding ROUGE gain is negligible ($+1.3\%$).
The pointwise GPT-class judge performs even worse, with Pearson $r = 0.099$ to ROUGE (Appendix~\ref{app:alternative_metrics}).
Both findings indicate that general-purpose LLMs cannot capture per-user stylistic preferences from limited history.
The trained probabilistic reward model is also $33\times$ faster per query.
Our diagnostic framework predicts these failures: with $\rho \approx 0$ between LLM scores and true rewards, Proposition~\ref{prop:unified_law} yields flat scaling regardless of $N$.

%==============================================================================

%==============================================================================

%==============================================================================

%==============================================================================

\subsection{Sensitivity Analyses and Statistical Stability}
\label{app:sensitivity}

\paragraph{Sensitivity to LoRA rank.}
\label{app:lora_rank}
A natural question is whether reducing LoRA rank could mitigate Deterministic User RM's collapse and hacking by limiting overfitting capacity.
Table~\ref{tab:lora_rank} sweeps the LoRA rank $r \in \{2, 4, 8, 16\}$ for Deterministic User RM on LaMP-4 and contrasts with Probabilistic User RM at the same $r{=}8$ baseline.

\begin{table}[h]
\centering
\caption{LoRA rank sweep on Deterministic User RM (LaMP-4).}
\label{tab:lora_rank}
\begin{tabular}{lccc}
\toprule
\textbf{Variant} & RM Pearson & ROUGE@$N{=}30$ & Gain over Random \\
\midrule
Det.\ $r{=}2$  & 0.733 & 0.136 & $+6.0\%$ \\
Det.\ $r{=}4$  & 0.735 & 0.138 & $+7.4\%$ \\
Det.\ $r{=}8$  & 0.790 & 0.139 & $+8.2\%$ \\
Det.\ $r{=}16$ & 0.758 & 0.138 & $+7.1\%$ \\
\midrule
\textbf{Prob.\ $r{=}8$} & \textbf{0.916} & \textbf{0.179} & $\mathbf{+39.3\%}$ \\
\bottomrule
\end{tabular}
\end{table}

Reducing rank does \emph{not} fix the underlying problem: even at $r{=}2$ (four-fold fewer parameters than $r{=}8$), Deterministic RM achieves only $+6.0\%$ over random, with $\alpha = 0.20$ and $\beta = 0.32$ unchanged.
At the same $r{=}8$, Probabilistic RM achieves $+39.3\%$.
The result indicates that the limitation of Deterministic User RM is not a capacity-driven overfitting issue but a property of the training dynamics under MSE loss with low-variance labels (Lemma~\ref{lem:gradient_buffering}).

%==============================================================================

%==============================================================================

\paragraph{Sensitivity to per-user data volume.}
\label{app:data_sensitivity}
We test whether TTP requires extensive per-user data by downsampling each user's reward-model training set on LaMP-4 to $\{10\%, 25\%, 50\%, 100\%\}$, leaving the validation split unchanged.
Table~\ref{tab:data_sensitivity} reports the resulting Pearson correlation of the trained Probabilistic User RM against ROUGE and the Best-of-$N$ ROUGE at several values of $N$.

\begin{table}[h]
\centering
\caption{Probabilistic User RM trained on a fraction of each user's history, evaluated on a fixed LaMP-4 validation set. The Random baseline is approximately $0.128$ across all $N$.}
\label{tab:data_sensitivity}
\begin{tabular}{lccccc}
\toprule
\textbf{Data \%} & RM Pearson & $N{=}5$ & $N{=}10$ & $N{=}15$ & $N{=}30$ \\
\midrule
10\%   & 0.671 & 0.134 & 0.133 & 0.133 & 0.130 \\
25\%   & 0.781 & 0.134 & 0.137 & 0.137 & 0.136 \\
50\%   & 0.887 & 0.138 & 0.143 & 0.145 & 0.141 \\
100\%  & 0.916 & ---   & ---   & ---   & \textbf{0.179} \\
\bottomrule
\end{tabular}
\end{table}

TTP delivers gains over Random even at $10\%$ data ($+3.9\%$ at $N{=}15$), with super-linear scaling in data volume: doubling from $25\%$ to $50\%$ nearly doubles the ROUGE gain over Random ($7.0\% \to 13.4\%$ at $N{=}15$).
This pattern is consistent with our predictive law: more training data raises $\bar{\rho}_{+}$ and reduces $\alpha$, both of which lift $\rho_{\text{eff}}$ and hence the scaling slope.

%==============================================================================

%==============================================================================

\paragraph{Statistical stability across repetitions.}
\label{app:error_bars}
The main results of Section~\ref{sec:experiments_main} are averages over $10$ independent repetitions, each sampling $N$ candidates from a pre-generated pool of size $40$.
Figure~\ref{fig:error_bars_lamp4} and Table~\ref{tab:error_bars} report mean$\pm$std at representative $N$ on LaMP-4 (RAG policy).
All standard deviations are below $0.007$, far smaller than the inter-method differences; pairwise differences between Probabilistic and the other methods are significant at $p < 0.01$ across all $N \geq 5$.
The selector ranking is consistent across all repetitions and all $N$, confirming the stability of the conclusions reported in the main text.

\begin{figure}[h]
    \centering
    \includegraphics[width=0.55\textwidth]{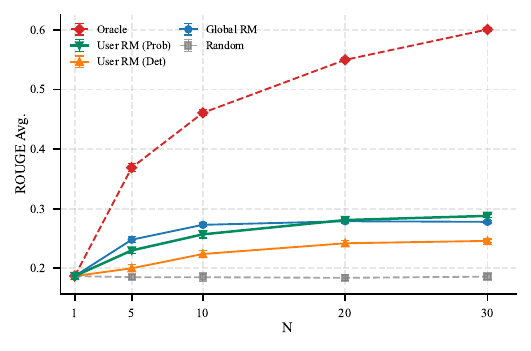}
    \caption{Best-of-$N$ ROUGE on LaMP-4 (RAG policy) with $1\sigma$ error bars over $10$ repetitions. Error bars are smaller than the marker for most points, indicating tight stability across runs.}
    \label{fig:error_bars_lamp4}
\end{figure}

\begin{table}[h]
\centering
\caption{Mean$\pm$std over $10$ repetitions on LaMP-4 (RAG policy, candidate pool size $40$). Absolute values use a different pool size from Section~\ref{sec:experiments_main}; the relative ranking and scaling pattern are identical.}
\label{tab:error_bars}
\small
\begin{tabular}{lccccc}
\toprule
$N$ & Random & Global & Det.\ User & \textbf{Prob.\ User} & Oracle \\
\midrule
1  & $0.187\pm0.005$ & $0.187\pm0.005$ & $0.187\pm0.005$ & $0.187\pm0.005$ & $0.187\pm0.005$ \\
5  & $0.185\pm0.004$ & $0.248\pm0.005$ & $0.200\pm0.006$ & $0.230\pm0.005$ & $0.369\pm0.007$ \\
10 & $0.185\pm0.007$ & $0.273\pm0.003$ & $0.224\pm0.006$ & $0.257\pm0.006$ & $0.461\pm0.005$ \\
20 & $0.184\pm0.005$ & $0.279\pm0.004$ & $0.242\pm0.005$ & $0.281\pm0.004$ & $0.550\pm0.003$ \\
30 & $0.186\pm0.005$ & $0.278\pm0.002$ & $0.246\pm0.003$ & $\mathbf{0.288\pm0.003}$ & $0.601\pm0.003$ \\
\bottomrule
\end{tabular}
\end{table}

%==============================================================================

%==============================================================================

%==============================================================================

%==============================================================================

\subsection{Variance-based Selection Strategies}
\label{app:variance_selection}

A natural question is whether the predicted variance from probabilistic User RM can improve candidate selection at inference time. We investigate several variance-based selection strategies and find that none outperform simple mean-based selection.

\paragraph{Motivation.}
The probabilistic User RM predicts both mean $\mu(x,y)$ and variance $\sigma^2(x,y)$ for each candidate. Intuitively, the variance captures the model's uncertainty about its prediction. One might expect that incorporating uncertainty could lead to better selection decisions, similar to exploration-exploitation trade-offs in bandit problems.

\paragraph{Strategies Evaluated.}
We evaluate the following variance-based selection strategies:

\begin{itemize}
    \item \textbf{Mean Only}: Select $\arg\max_y \mu(x,y)$ (baseline).
    \item \textbf{LCB-$\beta$} (Lower Confidence Bound): Select $\arg\max_y [\mu(x,y) - \beta \cdot \sigma(x,y)]$, preferring candidates with high mean and low uncertainty.
    \item \textbf{UCB-$\beta$} (Upper Confidence Bound): Select $\arg\max_y [\mu(x,y) + \beta \cdot \sigma(x,y)]$, allowing exploration of uncertain candidates.
    \item \textbf{Var-Filter-$p$}: Filter out candidates in the top $p$\% variance, then select by mean among remaining candidates.
    \item \textbf{SNR} (Signal-to-Noise Ratio): Select $\arg\max_y [\mu(x,y) / \sigma(x,y)]$.
\end{itemize}

\paragraph{Results.}
Table~\ref{tab:variance_strategies} shows the results on LaMP-4. All variance-based strategies perform worse than or equal to mean-only selection.

\begin{table}[h]
\centering
\caption{Comparison of variance-based selection strategies on LaMP-4. All strategies underperform simple mean-based selection, indicating that variance does not provide useful signal for candidate selection at inference time.}
\label{tab:variance_strategies}
\resizebox{0.65\textwidth}{!}{
\begin{tabular}{lcccccc}
\toprule
\textbf{Strategy} & $N$=1 & $N$=5 & $N$=10 & $N$=15 & $N$=20 & $N$=30 \\
\midrule
Oracle & .129 & .181 & .201 & .212 & .219 & .230 \\
\midrule
Mean Only & .129 & \textbf{.158} & \textbf{.166} & \textbf{.172} & \textbf{.174} & \textbf{.179} \\
\midrule
LCB-0.5 & .129 & \textbf{.158} & \textbf{.166} & .171 & .173 & .178 \\
LCB-1.0 & .129 & .156 & .165 & .169 & .171 & .175 \\
LCB-2.0 & .129 & .154 & .161 & .164 & .166 & .169 \\
\midrule
UCB-0.5 & .129 & \textbf{.158} & \textbf{.166} & \textbf{.172} & \textbf{.174} & \textbf{.179} \\
UCB-1.0 & .129 & .157 & .165 & .170 & .172 & .178 \\
UCB-2.0 & .129 & .155 & .162 & .167 & .169 & .174 \\
\midrule
Var-Filter-10\% & .129 & .152 & .163 & .166 & .169 & .173 \\
Var-Filter-20\% & .129 & .152 & .160 & .164 & .167 & .171 \\
Var-Filter-30\% & .129 & .146 & .157 & .160 & .164 & .168 \\
\midrule
SNR & .129 & .150 & .155 & .156 & .156 & .157 \\
\bottomrule
\end{tabular}}
\end{table}

\paragraph{Discussion.}
The failure of variance-based strategies suggests that while variance prediction is beneficial during training (as part of the NLL loss), the learned variance lacks proper calibration for use at inference time. Specifically:

\begin{itemize}
    \item \textbf{LCB strategies} assume that low variance indicates reliable predictions, but this relationship may not hold if variance is not well-calibrated.
    \item \textbf{UCB strategies} assume high variance candidates have potential upside, but without proper calibration, high variance may simply indicate poor predictions.
    \item \textbf{Variance filtering} removes potentially good candidates that happen to have high predicted variance.
    \item \textbf{SNR} is particularly sensitive to variance scale, performing worst among all strategies.
\end{itemize}

As our probabilistic User RM is not trained with variance-aware ranking objectives, there is no guarantee that predicted uncertainty aligns with selection quality. 
For high-quality candidates (i.e., those with large $\mu(x,y)$), the variance is often already low or relatively uniform. In this regime, ranking is dominated by differences in the mean, leaving little room for variance-aware strategies to meaningfully reorder candidates. This may explain why LCB/UCB with small $\beta$ values closely match mean-only selection, while larger $\beta$ values increasingly hurt performance.

%============================================================================

%============================================================================

%==============================================================================

%==============================================================================

\subsection{Computational Efficiency Analysis}
\label{app:compute-efficiency}

We measure the computational costs of training-based personalization and TTP to understand their efficiency characteristics.

\paragraph{Experimental Setup.}
We benchmark both approaches on LaMP-4 using an NVIDIA H100 80GB GPU. For training-based personalization, we fine-tune per-user Qwen3-4B policy models with LoRA adapters (rank 16, sequence length 256). For TTP, we train per-user probabilistic reward models based on Qwen2.5-1.5B (sequence length 128). We measure per-sample training time and per-query inference time across 10 users. Table~\ref{tab:unit-cost} reports the measured unit costs.

\begin{table}[h]
\centering
\caption{Measured computational costs for training and inference.}
\label{tab:unit-cost}
\begin{tabular}{lccc}
\toprule
\textbf{Component} & \textbf{Model} & \textbf{Unit Cost} \\
\midrule
Policy Training & Qwen3-4B & 59 ms/sample \\
RM Training & Qwen2.5-1.5B & 19 ms/sample \\
\midrule
Policy Inference (generation) & Qwen3-4B & 1498 ms/query \\
RM Inference (scoring) & Qwen2.5-1.5B & 2.3 ms/candidate \\
\bottomrule
\end{tabular}
\end{table}

\paragraph{Key Findings.}
Our measurements reveal two important characteristics:

\begin{enumerate}
    \item \textbf{Training efficiency}: The reward model trains $3.1\times$ faster per sample than the policy model (19 ms vs.\ 59 ms), owing to the smaller backbone (1.5B vs.\ 4B parameters) and shorter sequence length.
    
    \item \textbf{Inference cost structure}: The RM scoring cost is negligible compared to generation cost. Scoring a single candidate takes only 2.3 ms, while generating one response takes 1498 ms—a $650\times$ difference. Even with $N=30$ candidates, the total scoring time (69 ms) remains less than 5\% of a single generation.
\end{enumerate}

These findings suggest that the computational bottleneck of TTP lies entirely in candidate generation, not in reward model scoring. In deployment scenarios where candidates can be pre-generated, cached, or shared across users, TTP's inference overhead reduces to just the scoring component, making it highly efficient for high-throughput personalization.

%==============================================================================

%==============================================================================

\section{Theoretical Validation}
\label{app:theory_validation}

\subsection{Quantitative Validation of the Predictive Scaling Law}
\label{app:scaling_law_validation}

We provide the per-task evidence supporting the aggregate result reported in Section~\ref{sec:experiments_main}.
For each task, we calibrate the scale parameter $c\bar{\sigma}$ once from the Oracle scaling curve, measure $(\alpha, \beta, \bar{\rho}_{+}, \bar{\rho}_{-})$ from each reward model on the validation split, and substitute the values into Proposition~\ref{prop:unified_law} to predict the full Best-of-$N$ curve.
Table~\ref{tab:e1_per_task} reports the four diagnostic quantities, the effective correlation $\rho_{\text{eff}} = (1-\beta)\bar{\rho}_{+} - \beta|\bar{\rho}_{-}|$, and the relative mean absolute error (relMAE) between predicted and observed ROUGE for $N \in \{1, 5, 10, 15, 20, 30\}$.

\begin{table}[h]
\centering
\caption{Per-task validation of Proposition~\ref{prop:unified_law}. The four diagnostic quantities are measured directly from each reward model; relMAE compares predicted vs.\ observed Best-of-$N$ ROUGE.}
\label{tab:e1_per_task}
\small
\begin{tabular}{llcccccc}
\toprule
\textbf{Task} & \textbf{RM} & $\alpha$ & $\beta$ & $\bar{\rho}_{+}$ & $\bar{\rho}_{-}$ & $\rho_{\text{eff}}$ & \textbf{relMAE} \\
\midrule
LaMP-4    & Global & 0.33 & 0.45 & 0.222 & $-$0.214 & 0.017 & 8.8\% \\
          & Det.   & 0.20 & 0.33 & 0.273 & $-$0.202 & 0.093 & 1.4\% \\
          & Prob.  & 0.00 & 0.08 & 0.523 & $-$0.168 & 0.467 & 2.8\% \\
\midrule
LaMP-5    & Global & 0.17 & 0.46 & 0.300 & $-$0.261 & 0.035 & 7.4\% \\
          & Det.   & 0.07 & 0.37 & 0.312 & $-$0.266 & 0.090 & 1.8\% \\
          & Prob.  & 0.00 & 0.22 & 0.458 & $-$0.242 & 0.307 & 0.6\% \\
\midrule
Abstract  & Global & 0.00 & 0.37 & 0.323 & $-$0.254 & 0.110 & 0.4\% \\
          & Det.   & 0.00 & 0.17 & 0.496 & $-$0.295 & 0.363 & 1.6\% \\
          & Prob.  & 0.00 & 0.01 & 0.758 & $-$0.187 & 0.754 & 0.6\% \\
\midrule
Topic     & Global & 0.05 & 0.24 & 0.634 & $-$0.358 & 0.380 & 4.0\% \\
          & Det.   & 0.25 & 0.20 & 0.453 & $-$0.304 & 0.224 & 1.6\% \\
          & Prob.  & 0.00 & 0.01 & 0.787 & $-$0.223 & 0.772 & 1.1\% \\
\midrule
Product   & Global & 0.00 & 0.04 & 0.486 & $-$0.122 & 0.462 & 1.3\% \\
          & Det.   & 0.25 & 0.19 & 0.324 & $-$0.254 & 0.162 & 6.9\% \\
          & Prob.  & 0.00 & 0.00 & 0.757 & $-$0.213 & 0.757 & 3.4\% \\
\midrule
\textbf{Avg.} & Global & 0.11 & 0.31 & 0.39 & $-$0.24 & 0.20 & \textbf{4.4\%} \\
              & Det.   & 0.15 & 0.25 & 0.37 & $-$0.26 & 0.19 & \textbf{2.7\%} \\
              & Prob.  & \textbf{0.00} & \textbf{0.06} & \textbf{0.66} & $-$\textbf{0.21} & \textbf{0.61} & \textbf{1.7\%} \\
\bottomrule
\end{tabular}
\end{table}

Across all $15$ (task, RM) configurations, the formula achieves average relMAE of $2.9\%$ with $R^{2} > 0.9$ in $13$ of $15$ cases.
The two cases with $R^{2} < 0.9$ are LaMP-4 Global RM and LaMP-5 Global RM, where extreme reward hacking ($\rho_{\text{eff}} \approx 0$) yields nearly flat curves whose remaining variance is dominated by sampling noise rather than systematic scaling.

\begin{figure}[h]
    \centering
    % Source figure: rebuttal_e1_e10/outputs/scaling_law_validation.pdf
    \includegraphics[width=0.95\textwidth]{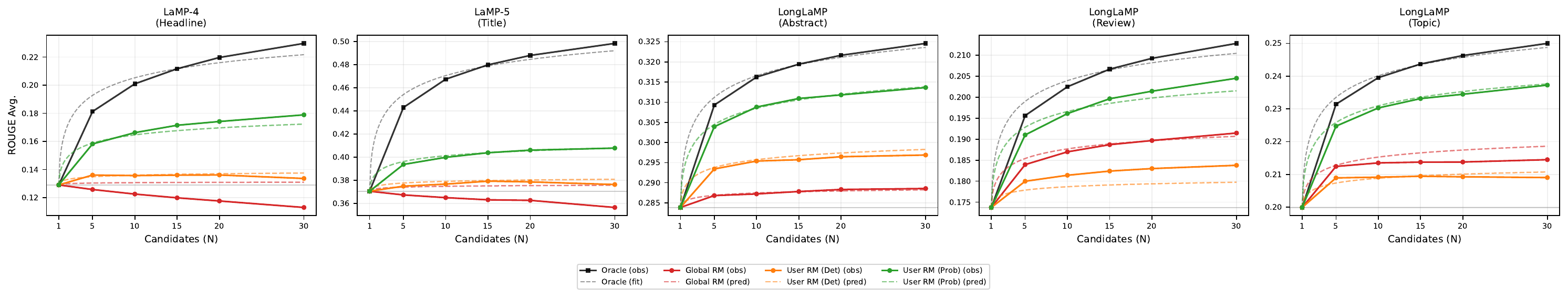}
    \caption{Predicted (dashed) vs.\ observed (solid) Best-of-$N$ ROUGE for all $15$ (task, RM) configurations. Predictions use the four diagnostic quantities from Table~\ref{tab:e1_per_task} substituted into Proposition~\ref{prop:unified_law}; the scale parameter $c\bar{\sigma}$ is calibrated once per task from Oracle and shared across all reward models.}
    \label{fig:scaling_law_validation}
\end{figure}

%==============================================================================

%==============================================================================

%==============================================================================

%==============================================================================

\subsection{Validating the Theoretical Assumptions}
\label{app:assumption_validation}

Our derivations rely on two analytical assumptions: (i) the per-user true-reward distribution is sub-Gaussian (used in Theorem~\ref{thm:oracle_scaling} and inherited by Lemma~\ref{lem:corr_scaling}), and (ii) the joint distribution of predicted and true rewards has approximately linear conditional expectation (used in the Bivariate-Gaussian-style derivation of Lemma~\ref{lem:corr_scaling}).
Below we empirically verify that both assumptions hold to a useful degree on our data.

\paragraph{Sub-Gaussian property of true rewards.}
For each user we fit the smallest sub-Gaussian parameter $\sigma$ that satisfies the moment-generating-function bound and compare the resulting tail behavior against the empirical reward distribution.
We illustrate this on LaMP-4 in Figure~\ref{fig:sub_gaussian_verification}: the empirical CDF lies close to the Gaussian envelope (small deviations only in the extreme tails), consistent with bounded-but-light-tailed ROUGE-derived rewards; the same pattern holds on the other four tasks.
Although the bound becomes loose at the very tail, this only inflates the constant in front of $\sqrt{\ln N}$ and does not change the scaling order.

\begin{figure}[h]
    \centering
    \includegraphics[width=0.95\textwidth]{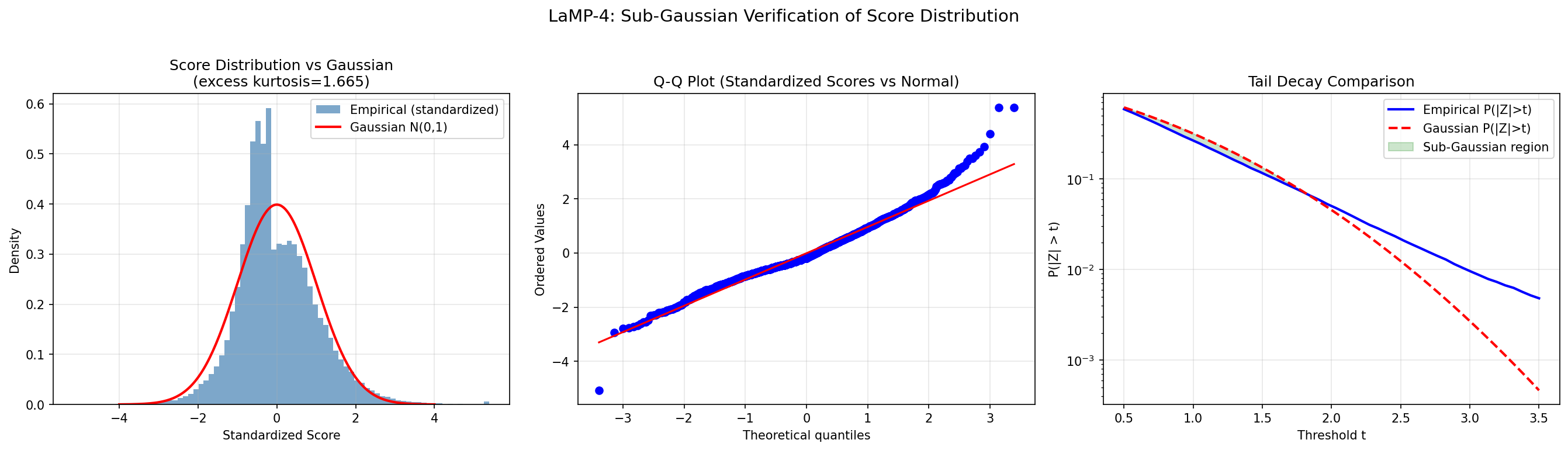}
    \caption{Sub-Gaussian property check on a representative LaMP-4 user. The empirical reward CDF (blue) is overlaid on the best-fitting Gaussian envelope (orange). Deviations are small except in the extreme tails, supporting the sub-Gaussian assumption used in Theorem~\ref{thm:oracle_scaling}.}
    \label{fig:sub_gaussian_verification}
\end{figure}

\paragraph{Linear conditional expectation of true reward given predicted reward.}
Lemma~\ref{lem:corr_scaling} relies on $\mathbb{E}[r^{*} \mid \hat{r}]$ being approximately linear in $\hat{r}$ (this is exact under joint Normality and weakly required under any second-order linear regression structure).
We illustrate this on LaMP-4 in Figure~\ref{fig:linear_conditional_expectation}, binning predictions $\hat{r}$ and plotting the conditional mean of $r^{*}$ against the bin centre for each of the three reward-model variants.
The binned conditional means lie close to a straight line whose slope matches the empirical correlation $\rho$ used in our analysis.
Departures from linearity are small and concentrated at the rare extremes, consistent with the remark in Appendix~\ref{app:proof_correlation} that high-score regions slightly under-perform the global correlation.

\begin{figure}[h]
    \centering
    \includegraphics[width=0.95\textwidth]{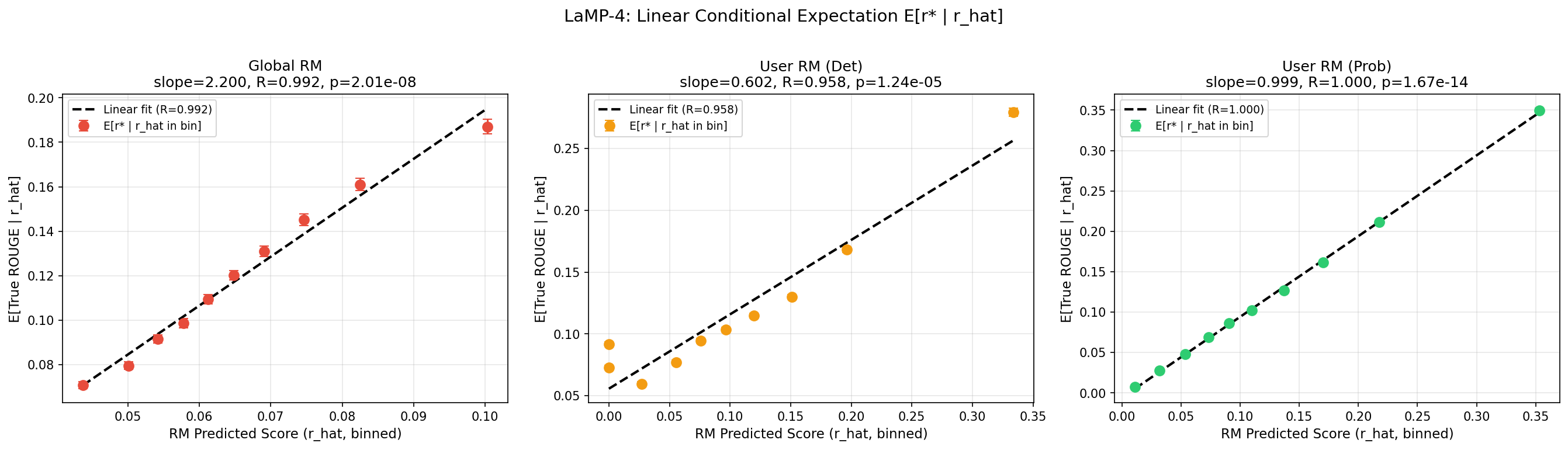}
    \caption{Bivariate linearity check on LaMP-4: empirical $\mathbb{E}[r^{*} \mid \hat{r}]$ vs.\ predicted reward $\hat{r}$, computed by binning the predictions of each reward model (Global, Det, Prob). The binned conditional means track a straight line whose slope matches the empirical correlation $\rho$, supporting the linear-conditional-expectation assumption in Lemma~\ref{lem:corr_scaling}.}
    \label{fig:linear_conditional_expectation}
\end{figure}

Together, these checks ensure that the formal derivations in Appendices~\ref{app:proof_oracle} and \ref{app:proof_correlation} are not artefacts of overly idealised assumptions.

%==============================================================================

%==============================================================================

\subsection{Detailed Failure-Mode Analysis}
\label{app:failure_mode_multitask}

\paragraph{Cross-task user-level collapse.}
\label{app:user-collapse-multitask}
\begin{figure*}[t]
    \centering
    \includegraphics[width=\textwidth]{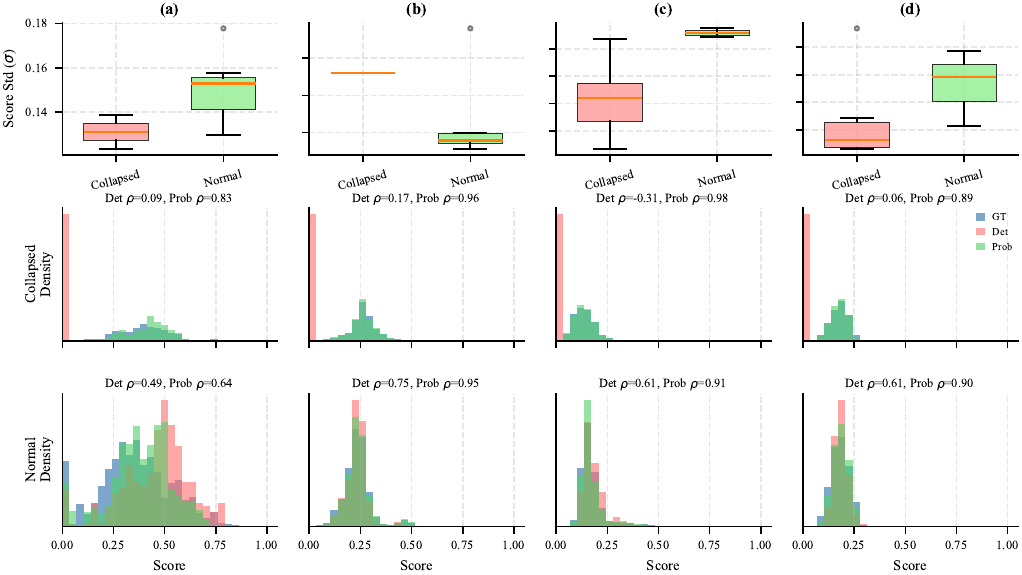}
        \caption{User-level collapse analysis across four tasks: (a) LaMP-5, (b) LongLaMP-Abstract, (c) LongLaMP-Product Review, (d) LongLaMP-Topic Writing. Top row: score standard deviation separating collapsed vs normal users. Middle/Bottom rows: score distributions for collapsed and normal users respectively, with Pearson correlations for deterministic (Det) and probabilistic (Prob) RMs.}
    \label{fig:user-analysis-multitask}
\end{figure*}

Figure~\ref{fig:user-analysis-multitask} extends our user-level collapse analysis to four additional tasks. The results consistently support our theoretical findings across diverse personalization scenarios.

\textbf{Collapse Identification.} Across all tasks, we observe a clear separation in score standard deviation ($\sigma$) between collapsed and normal users. LongLaMP-Product Review (c) exhibits the most pronounced separation, with collapsed users showing extremely low variance ($\sigma \approx 0.13$) compared to normal users ($\sigma \approx 0.17$). This aligns with the task's nature where deterministic models tend to predict similar review scores regardless of product differences.

\textbf{Deterministic RM Failure on Collapsed Users.} The middle row demonstrates that deterministic RMs consistently fail on collapsed users, producing near-degenerate predictions concentrated at a single value. This is most severe in LongLaMP-Product Review (c), where the deterministic RM achieves a \emph{negative} correlation ($\rho = -0.31$) on collapsed users, indicating predictions that are worse than random. In contrast, the probabilistic RM maintains strong correlations even for collapsed users: $\rho = 0.83$ (LaMP-5), $\rho = 0.96$ (Abstract), $\rho = 0.98$ (Product Review), and $\rho = 0.89$ (Topic Writing).

\textbf{Robust Performance on Normal Users.} For normal users (bottom row), both RMs perform reasonably well, though the probabilistic RM consistently achieves higher correlations. Notably, LongLaMP-Abstract (b) shows the smallest gap between methods on normal users ($\rho = 0.75$ vs $0.95$), suggesting that academic writing styles may be more predictable, reducing the advantage of uncertainty modeling for well-behaved users.

%==============================================================================

%==============================================================================

\paragraph{Cross-task query-level reward hacking.}
\label{app:query-hacking-multitask}
\begin{figure*}[t]
    \centering
    \includegraphics[width=\textwidth]{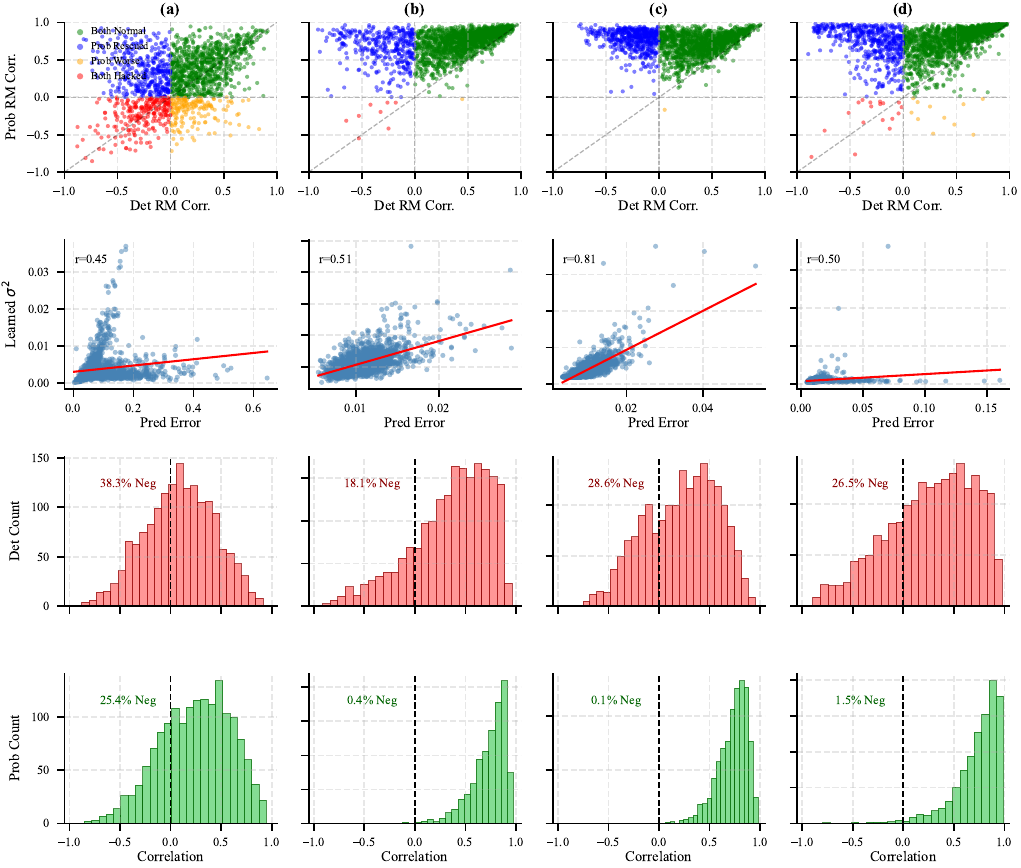}
    \caption{Query-level hacking analysis across four tasks: (a) LaMP-5, (b) LongLaMP-Abstract, (c) LongLaMP-Product Review, (d) LongLaMP-Topic Writing. Row 1: Per-query correlation comparison between Det and Prob RMs. Row 2: Learned variance vs prediction error. Rows 3--4: Distribution of per-query correlations for Det and Prob RMs, with percentage of negative correlations indicated.}
    \label{fig:query-analysis-multitask}
\end{figure*}

Figure~\ref{fig:query-analysis-multitask} presents query-level analysis across four additional tasks, validating the generality of our uncertainty-aware selection mechanism.

\textbf{Correlation Between Uncertainty and Error.} The second row shows strong positive correlations between learned variance $\sigma^2$ and prediction error across all tasks. LongLaMP-Product Review (c) achieves the highest correlation ($r = 0.81$), indicating that the probabilistic RM learns particularly well-calibrated uncertainty estimates for this task. The other tasks show moderate correlations ($r = 0.45$--$0.51$), confirming that the model successfully learns to express higher uncertainty for difficult queries.

\textbf{Dramatic Reduction in Negative Correlations.} The bottom two rows compare the distribution of per-query correlations. The deterministic RM exhibits substantial negative correlation rates: 38.3\% (LaMP-5), 18.1\% (Abstract), 28.6\% (Product Review), and 26.5\% (Topic Writing). These represent queries where the RM's ranking is \emph{inversely} related to true quality, a severe failure mode for Best-of-N selection.

The probabilistic RM dramatically reduces these failure cases: negative correlations drop to 25.4\% (LaMP-5), 0.4\% (Abstract), 0.1\% (Product Review), and 1.5\% (Topic Writing). The improvement is most striking for the LongLaMP tasks, where negative correlations are nearly eliminated ($<2\%$).

\textbf{Task-Specific Observations.} LaMP-5 (a) shows the smallest improvement (38.3\% $\to$ 25.4\%), likely because title generation from abstracts has inherent ambiguity that even uncertainty-aware methods cannot fully resolve. In contrast, LongLaMP tasks benefit more substantially, possibly because longer outputs provide richer signals for uncertainty estimation.

\textbf{Rescue Analysis.} The scatter plots (top row) reveal that the probabilistic RM ``rescues'' many queries from the hacking region (blue points above the diagonal), while rarely degrading performance (orange points below diagonal are sparse). This asymmetry confirms that uncertainty-aware selection provides a safety mechanism against reward hacking without sacrificing performance on well-predicted queries.

%==============================================================================

%==============================================================================

\paragraph{A priori predictability.}
The analyses in Appendices~\ref{app:user-collapse-multitask} and~\ref{app:query-hacking-multitask} document \emph{which} users collapse and \emph{which} queries get hacked.
This subsection asks the stronger question: can these failures be predicted \emph{a priori} from training-data statistics, before running any Best-of-$N$ experiment?

\paragraph{Predicting user-level collapse.}
Users with low ground-truth ROUGE variance are predicted to collapse more often, since narrow label distributions provide weak discriminative signal under MSE training (Lemma~\ref{lem:gradient_buffering}).
Table~\ref{tab:collapse_predictability} verifies this on all five tasks: low-variance users collapse $1\times$--$13\times$ more often than high-variance users under Deterministic RM, while Probabilistic RM eliminates collapse entirely on every task regardless of label variance.

\begin{table}[h]
\centering
\caption{User-level collapse rates (fraction of users with $\rho_{u} < 0.1$), split by ground-truth ROUGE variance (low vs.\ high std halves). Probabilistic User RM achieves $0\%$ collapse on every task.}
\label{tab:collapse_predictability}
\begin{tabular}{lcccccc}
\toprule
& \multicolumn{2}{c}{Det.\ User RM} & \multicolumn{2}{c}{\textbf{Prob.\ User RM}} & Det.\ ratio \\
\cmidrule(lr){2-3} \cmidrule(lr){4-5}
\textbf{Task} & Low std & High std & Low std & High std & (low/high) \\
\midrule
LaMP-4   & 33.3\% & 6.7\%  & 0\% & 0\% & $5\times$ \\
LaMP-5   & 13.3\% & 0.0\%  & 0\% & 0\% & $>13\times$ \\
Abstract & 0.0\%  & 0.0\%  & 0\% & 0\% & --- \\
Topic    & 25.0\% & 25.0\% & 0\% & 0\% & $1\times$ \\
Product  & 30.0\% & 10.0\% & 0\% & 0\% & $3\times$ \\
\bottomrule
\end{tabular}
\end{table}

\paragraph{Predicting query-level hacking.}
Queries with low candidate-quality spread are predicted to hack more often, since limited training signal drives the model toward spurious patterns (Lemma~\ref{lem:implicit_regularization}).
Table~\ref{tab:hacking_predictability} verifies this prediction.
Across the population of queries, Probabilistic RM reduces hacking $4$--$8\times$ relative to Global RM, with the largest gains precisely on the predicted high-risk queries (low candidate-quality spread).

\begin{table}[h]
\centering
\caption{Query-level hacking rates (fraction of queries with $\rho_{q} < 0$), split by candidate-quality spread (low vs.\ high std halves). Selected (task, RM) pairs.}
\label{tab:hacking_predictability}
\begin{tabular}{llccc}
\toprule
\textbf{Task} & \textbf{RM} & Low spread & High spread & Reduction \\
\midrule
Topic     & Global & 43.2\% & 7.8\%  & $5.5\times$ \\
Product   & Det.   & 42.2\% & 19.8\% & $2.1\times$ \\
LaMP-4    & Det.   & 34.8\% & 31.4\% & $1.1\times$ \\
\midrule
Abstract  & Prob.  & 1.2\%  & 0.0\%  & --- \\
Product   & Prob.  & $<$0.1\% & $<$0.1\% & --- \\
LaMP-4    & Prob.  & 10.9\% & 5.2\%  & $2.1\times$ \\
\bottomrule
\end{tabular}
\end{table}

These two tables transform the diagnostic framework into an actionable tool: practitioners can identify high-risk users and queries by computing simple per-user and per-query label statistics on the training set, before any Best-of-$N$ experiment is run.

%==============================================================================

%==============================================================================

\section{Pipeline and Implementation Details}
\label{app:implementation}

Our experimental pipeline consists of three stages, as illustrated in Figure~\ref{fig:overview}: (1) \textbf{Personalized Prompting}, where we construct personalized policy models using retrieval-augmented generation; (2) \textbf{Reward Model Training}, where we train user-specific reward models to capture individual preferences; and (3) \textbf{Inference}, where we apply test-time personalization by sampling multiple candidates and selecting the best one via the reward model. We describe each stage in detail below.

\begin{figure*}[h]
    \centering
    \includegraphics[width=0.6\textwidth]{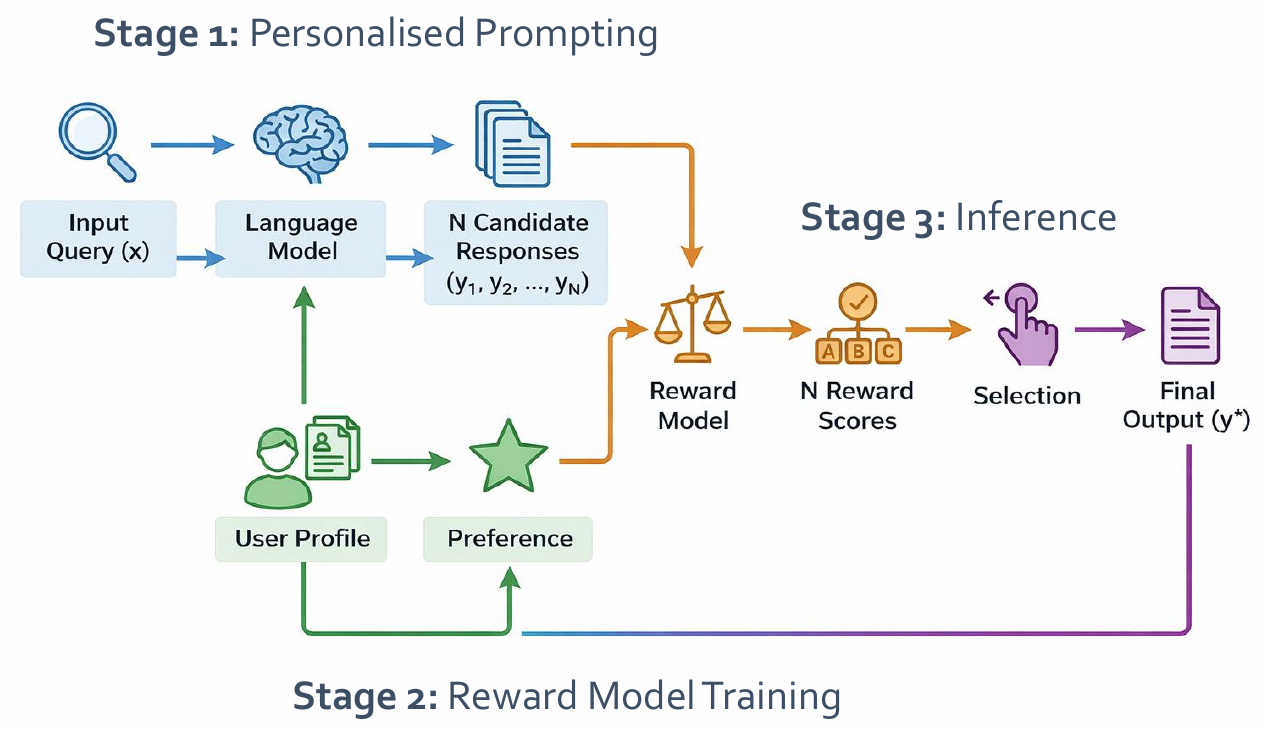}
    \caption{Overview of the experimental pipeline.}
    \label{fig:overview}
\end{figure*}

%==============================================================================

%==============================================================================

\subsection{Stage 1: Personalized Prompting}
\label{app:stage1}

The goal of personalized prompting is to generate candidate responses that serve as training data for reward models. This stage involves dataset preparation and policy model configuration.

\paragraph{Datasets.}
We evaluate on five personalized text generation tasks: \textbf{News Headline Generation} (LaMP-4) and \textbf{Scholarly Title Generation} (LaMP-5) from LaMP~\citep{salemi-etal-2024-lamp}, and \textbf{Abstract Generation}, \textbf{Product Review Writing}, and \textbf{Topic Writing} from LongLaMP~\citep{kumar2024longlamp}. We exclude email-related tasks from both benchmarks as they require additional access to private email datasets. For all tasks, we adopt the \emph{temporal-based} data split provided by the original benchmarks, which ensures that test queries are chronologically after training examples to simulate realistic deployment scenarios.

Following \citet{tan-etal-2024-democratizing}, we select the top-30 users with the longest profiles for LaMP tasks and top-20 users for LongLaMP tasks. Representative input-output examples for each task are shown as follows.

\begin{tcolorbox}[title=\textbf{LaMP-4: News Headline Generation}]
\textbf{Input:} Generate a headline for the following article: According to AP, Harry's days as a pilot are over. The prince, who spent over three years training as an Apache helicopter...

\textbf{Output:} Prince Harry Scores A New Job, Says Goodbye To His Pilot Uniform
\end{tcolorbox}

\begin{tcolorbox}[title=\textbf{LaMP-5: Scholarly Title Generation}]
\textbf{Input:} Generate a title for the following abstract of a paper: We propose a novel evolutionary approach to the problem of timing-driven FPGA placement. The method used is evolutionary programming (EP) with incremental position encoded in the population...

\textbf{Output:} An evolutionary approach to timing driven FPGA placement
\end{tcolorbox}

\begin{tcolorbox}[title=\textbf{LongLaMP: Abstract Generation}]
\textbf{Input:} Generate an abstract for the title ``Learned Token Pruning for Transformers'' using the following items: 1. Transformer models 2. Learned Token Pruning (LTP) 3. Performance improvement 4. FLOPs Reduction 5. Robustness

\textbf{Output:} Efficient deployment of transformer models in practice is challenging due to their inference cost including memory footprint, latency, and power consumption, which scales quadratically with input sequence length. To address this, we present a novel token reduction method dubbed Learned Token Pruning (LTP)...
\end{tcolorbox}

\begin{tcolorbox}[title=\textbf{LongLaMP: Product Review Writing}]
\textbf{Input:} Generate the review text for a product with rating ``5.0'' and description ``Kate Aster is the bestselling author of the Special Ops: Homefront series...''. Summary: ``Contract with a SEAL''.

\textbf{Output:} It is impossible not to fall in love with the characters in this book, which brings in strong, supportive relationships with family and friends. Aster has done a remarkable job of making the plot and characters engaging...
\end{tcolorbox}

\begin{tcolorbox}[title=\textbf{LongLaMP: Topic Writing}]
\textbf{Input:} Generate the content for a reddit post titled ``good riddance to bad rubbish''.

\textbf{Output:} As an Eagle who worked very hard to encourage BSA to change it's policy, I have to say this is good news. This is symbolic of the total defeat of the Christianist/anti-gay lobby in Scouting...
\end{tcolorbox}

\paragraph{Policy Model Configuration.}
We use \texttt{Qwen3-4B-Instruct}~\citep{yang2025qwen3} as the policy model backbone. To achieve personalization, we adopt a retrieval-augmented generation (RAG) approach with dense retrieval~\cite{salemi-etal-2024-lamp}. Specifically, we use \texttt{all-MiniLM-L6-v2}~\cite{reimers-gurevych-2019-sentence} as the text encoder for embedding-based retrieval. Given an input query, we encode it and retrieve the top-$k$ most similar examples from the user's profile based on cosine similarity. The retrieved examples are concatenated as context in the prompt.

For decoding, we use \texttt{vLLM}~\citep{kwon2023vllm} to accelerate generation. For each query, we sample $N=30$ candidate responses to construct reward model training data. 
Task-specific prompt templates are provided in Appendix~\ref{app:policy_templates}.

\paragraph{Generation Hyperparameters.}
Table~\ref{tab:dataset_stats} summarizes the generation hyperparameters and dataset statistics for each task.

\begin{table}[h]
\centering
\caption{Dataset statistics and generation hyperparameters. Avg Len denotes the average candidate length in characters.}
\label{tab:dataset_stats}
\resizebox{\columnwidth}{!}{%
\begin{tabular}{lcccccccc}
\toprule
\textbf{Task} & \textbf{\# Users} & \textbf{\# Queries} & \textbf{Cands/Query} & \textbf{Avg Len} & \textbf{Temp} & \textbf{Top-$k$} & \textbf{Top-$p$} & \textbf{\# History} \\
\midrule
LaMP-4 (Headline) & 30 & 24015 & 24.9 & 74 & 1.7 & 50 & 0.90 & 5 \\
LaMP-5 (Title) & 30 & 12036 & 16.8 & 96 & 1.7 & 50 & 0.90 & 5 \\
LongLaMP (Abstract) & 20 & 22388 & 30.0 & 1227 & 1.5 & 40 & 0.90 & 3 \\
LongLaMP (Review) & 20 & 12580 & 30.0 & 514 & 1.5 & 40 & 0.90 & 3 \\
LongLaMP (Topic) & 20 & 7123 & 29.2 & 1101 & 1.5 & 40 & 0.90 & 3 \\
\bottomrule
\end{tabular}
}
\end{table}

We set the number of candidate generations to $N=30$ for all tasks. However, for short-form generation tasks (LaMP-4 and LaMP-5), the policy model occasionally produces duplicate outputs despite temperature sampling. We remove exact duplicates from the candidate set, resulting in an average of fewer than 30 unique candidates per query for these tasks. Long-form generation tasks (LongLaMP) exhibit minimal duplication due to their greater output diversity.

%==============================================================================

%==============================================================================

\subsection{Stage 2: Reward Model Training}
\label{app:stage2}

This stage describes the training procedure for three reward model variants: Global RM, Deterministic User RM, and Probabilistic User RM.

\paragraph{Model Architecture.}
All reward models share the same backbone architecture: \texttt{Qwen2.5-1.5B-Instruct}~\citep{qwen2025qwen25} with masked mean pooling over hidden states.
For probabilistic User RM, an additional variance head predicts $\log \sigma^2$ with output passed through softplus and clamped to $[10^{-4}, 0.5]$.

\paragraph{Reward Label Construction.}
For each candidate response, we compute the ROUGE score against the ground-truth output as the reward label: $r = (\text{ROUGE-1} + \text{ROUGE-L})/2$.
where ROUGE-1 measures unigram overlap and ROUGE-L measures longest common subsequence. Scores are clipped to $[0, 1]$.

\paragraph{Training Configurations.}
Table~\ref{tab:rm_training} summarizes the training hyperparameters for each reward model variant.

\begin{table}[h]
\centering
\caption{Reward model training configurations.}
\label{tab:rm_training}
\begin{tabular}{lccc}
\toprule
\textbf{Hyperparameter} & \textbf{Global RM} & \textbf{User RM (Det)} & \textbf{User RM (Prob)} \\
\midrule
Backbone & \multicolumn{3}{c}{Qwen2.5-1.5B-Instruct} \\
Backbone training & LoRA & LoRA & LoRA \\
Learning rate (head) & $10^{-4}$ & $10^{-4}$ & $2 \times 10^{-4}$ \\
Learning rate (LoRA) & $10^{-4}$ & $10^{-4}$ & $2 \times 10^{-5}$ \\
Batch size & 48 & 16 & 8 \\
Epochs & 8 & 10 & 15 \\
Warmup ratio & 0.05 & 0.1 & 0.1 \\
Early stopping & -- & -- & patience=5 \\
Weight decay & 0.01 & 0.01 & 0.01 \\
Max length (LaMP) & 128 & 128 & 128 \\
Max length (LongLaMP) & 832 & 832 & 832 \\
\bottomrule
\end{tabular}
\end{table}

\paragraph{LoRA Configuration.}
For models using LoRA fine-tuning, we apply low-rank adaptation with rank $r=8$, scaling factor $\alpha=16$, and dropout $0.05$. Target modules are \texttt{q\_proj} and \texttt{v\_proj}. This results in approximately 4.6M trainable parameters per user model.

\paragraph{Loss Functions.}
\textbf{Deterministic User RM} is trained with MSE loss. 
\textbf{Probabilistic User RM} is trained with Gaussian NLL loss (Equation~\ref{eq:nll_loss}).
Two variants of user RM are also trained with a contrastive term for high-score discrimination, which the loss weighted parameter is $\lambda = 1.0$.
The contrastive loss operates on sample pairs where both have ground-truth scores above 0.5, with margin $m = 0.02$. We also apply linear sample weighting with weights in $[0.3, 1.0]$ proportional to ground-truth scores, emphasizing high-quality samples. \textbf{Global RM} is trained with BCE loss with sample reweighting for extreme scores.

\paragraph{User Selection and Data Split.}
For LaMP tasks, we select the top-30 users with the most data, using up to 3,000 samples per user. For LongLaMP tasks, we select the top-20 users with up to 5,000 samples per user. All data is split 80\%/20\% for training/validation per user.

\paragraph{Implementation Details.}
All experiments are conducted on a single NVIDIA H100 PCIe GPU. We use BF16 mixed precision training with Flash Attention 2~\citep{dao2024flashattention} for efficiency. The optimizer is AdamW~\citep{loshchilov2018decoupled} for all models.
All RMs share the same prompt templates, which are provided in Appendix~\ref{app:reward_templates}.
No retrieval-based method is used for RM training.

%==============================================================================

%==============================================================================

\subsection{Stage 3: Inference}
\label{app:stage3}

This stage describes the Test-Time Personalization (TTP) inference procedure and evaluation protocol.

\paragraph{TTP Inference Procedure.}
Given a test query $x$ for user $u$, TTP operates as follows:
\begin{enumerate}
    \item Sample $N$ candidates from the pre-generated candidate pool (up to 30 candidates per query).
    \item Score each candidate using the user-specific reward model $R_u$.
    \item Select the candidate with the highest predicted reward: $\hat{y} = \arg\max_{y \in \mathcal{Y}_N} R_u(x, y)$.
\end{enumerate}
For probabilistic User RM, we use only the predicted mean $\mu(x, y)$ for selection.

\paragraph{Evaluation Protocol.}
Table~\ref{tab:inference_config} summarizes the inference configuration. For each (user, query, $N$) combination, we repeat the random candidate sampling 3 times and report the average to reduce variance from candidate selection.

\begin{table}[h]
\centering
\caption{Inference and evaluation configuration.}
\label{tab:inference_config}
\begin{tabular}{lc}
\toprule
\textbf{Parameter} & \textbf{Value} \\
\midrule
$N$ values tested & 1, 5, 10, 20, 30 \\
Queries per user & 100 \\
Random trials per $N$ & 3 \\
Inference batch size & 32 \\
\bottomrule
\end{tabular}
\end{table}

%==============================================================================

%==============================================================================

\section{Prompt Templates}

\subsection{Policy Model Templates}
\label{app:policy_templates}

We adopt a retrieval-augmented generation (RAG) approach for personalized generation. For each query, we retrieve relevant examples from the user's history and construct the prompt using task-specific templates. Below we present the prompt templates for each task.

\begin{tcolorbox}[colback=blue!3, colframe=blue!40, title=\textbf{LaMP-4: News Headline Generation},
width=\textwidth]
\begin{verbatim}
{history}

### Instruction: Generate a headline for the following article. Match 
the user's headline style.

### Article:
{input}

### Headline:
\end{verbatim}
\end{tcolorbox}

\begin{tcolorbox}[colback=blue!3, colframe=blue!40, title=\textbf{LaMP-5: Scholarly Title Generation}]
\begin{verbatim}
{history}

### Instruction: Generate a title for the following abstract. Match the 
user's academic titling style. Output ONLY the title.

### Abstract:
{input}

### Title:
\end{verbatim}
\end{tcolorbox}

\begin{tcolorbox}[colback=blue!3, colframe=blue!40, title=\textbf{LongLaMP: Abstract Generation}]
\begin{verbatim}
{history}

### Instruction: Write a scientific abstract for the following title,  
matching the user's academic writing style found in the history.

### Title:
{input}

### Abstract:
\end{verbatim}
\end{tcolorbox}

\begin{tcolorbox}[colback=blue!3, colframe=blue!40, title=\textbf{LongLaMP: Product Review Writing}]
\begin{verbatim}
{history}

### Instruction:
You are the user described in the [User History] above. The User Request 
below specifies the product description, the rating you gave, and the 
summary you wrote. Write the full Review Text that corresponds to these 
details. Mimic your detailed, narrative writing style found in the 
history. Write a full body review (around 50-100 words), telling the 
story of your experience.

### User Request:
{input}

### Review Content:
\end{verbatim}
\end{tcolorbox}

\begin{tcolorbox}[colback=blue!3, colframe=blue!40, title=\textbf{LongLaMP: Topic Writing}]
\begin{verbatim}
{history}

### Instruction:
You are the user described in the [User History] above. Write a Reddit 
post about the topic below. Mimic the user's vocabulary and tone, but 
keep the response relatively short and punchy (around 100-200 words). 
Do not write a wall of text.

### Topic:
{input}

### Post Content:
\end{verbatim}
\end{tcolorbox}

In all templates, \texttt{\{history\}} is replaced with retrieved user history examples, and \texttt{\{input\}} is replaced with the current query. The model generates text following the final prompt marker (e.g., \texttt{\#\#\# Headline:}).

\subsection{Reward Model Prompt}
\label{app:reward_templates}

The reward model receives a structured prompt consisting of a task-specific instruction, the input query, and the candidate response:

\begin{tcolorbox}[colback=green!5, colframe=green!50!black, title=\textbf{Reward Model Prompt Template}]
\begin{verbatim}
[Instruction]: {task_instruction}
[Input]: {query}
[Document]: {candidate}
\end{verbatim}
\end{tcolorbox}

\noindent where \texttt{\{task\_instruction\}} is one of the following task-specific instructions:

\begin{tcolorbox}[colback=green!5, colframe=green!50!black, title=\textbf{LaMP-4: News Headline Generation}]
\small\texttt{Judge how good the Document is as a headline for the Article. Higher score = better.}
\end{tcolorbox}

\begin{tcolorbox}[colback=green!5, colframe=green!50!black, title=\textbf{LaMP-5: Scholarly Title Generation}]
\small\texttt{Judge how good the Document is as a title for the Abstract. Higher score = better.}
\end{tcolorbox}

\begin{tcolorbox}[colback=green!5, colframe=green!50!black, title=\textbf{LongLaMP: Abstract Generation}]
\small\texttt{Judge how good the Document is as an abstract for the Title. Higher score = better.}
\end{tcolorbox}

\begin{tcolorbox}[colback=green!5, colframe=green!50!black, title=\textbf{LongLaMP: Product Review Writing}]
\small\texttt{Judge how good the Document is as a review for the Product. Higher score = better.}
\end{tcolorbox}

\begin{tcolorbox}[colback=green!5, colframe=green!50!black, title=\textbf{LongLaMP: Topic Writing}]
\small\texttt{Judge how good the Document is as a post for the Topic. Higher score = better.}
\end{tcolorbox}
%---- END sections/appendix.tex ----

\end{document}